\documentclass[11pt]{article}

\usepackage{fullpage}
\usepackage[numbers,compress]{natbib}

\usepackage[utf8]{inputenc} 
\usepackage[T1]{fontenc}    
\usepackage{hyperref}       
\usepackage{url}            
\usepackage{booktabs}       
\usepackage{amsfonts}       
\usepackage{nicefrac}       
\usepackage{microtype}      
\usepackage{xcolor}         

\usepackage{natbib}
\usepackage{algorithm}
\usepackage[noend]{algorithmic}
\usepackage{amsmath,amsfonts,amssymb}
\usepackage{amssymb,amsmath,amsthm, amsfonts}
\usepackage{hyperref}
\usepackage{color}
\usepackage{mathrsfs}
\usepackage{enumitem}
\usepackage{bm}
\usepackage{multirow}
\usepackage{booktabs}
\usepackage{makecell}
\usepackage{graphicx}
\usepackage{caption}
\usepackage{thmtools}
\usepackage{thm-restate}
\usepackage{setspace}
\usepackage{makecell}
\definecolor{light-gray}{gray}{0.85}
\usepackage{colortbl}
\usepackage{xcolor}
\usepackage{hhline}
\allowdisplaybreaks

\hypersetup{
    colorlinks,
    linkcolor={blue!50!black},
    citecolor={blue!50!black},
}
\colorlet{linkequation}{blue}

  {\list{}{\leftmargin=0.3in\rightmargin=0.3in}\item[]}%
  {\endlist}

\newcommand{\argmin}{\mathop{\rm argmin}}
\newcommand{\argmax}{\mathop{\rm argmax}}

\newcommand{\poly}{\mathrm{poly}}

\newcommand{\norm}[1]{\|{#1} \|}
\newcommand{\bignorm}[1]{\Big\|{#1} \Big\|}

\newcommand{\E}{\mathbb{E}}
\renewcommand{\Pr}{\mathbb{P}}

\newcommand{\br}{{\mathbf r}}

\newcommand{\bV}{{\mathbf V}}
\newcommand{\bv}{\mathbf v}
\newcommand{\bQ}{{\mathbf Q}}
\newcommand{\bx}{{\mathbf x}}
\newcommand{\bX}{{\mathbf X}}
\newcommand{\bw}{{\mathbf w}}
\newcommand{\bmu}{{\boldsymbol{\mu}}}
\newcommand{\tbw}{\widetilde{\bw}}
\newcommand{\hbw}{\widehat{\bw}}
\newcommand{\bbw}{\overline{\bw}}
\newcommand{\btheta}{{\bm{\theta}}}

\newcommand{\bphi}{\boldsymbol{\phi}}

\newcommand{\cC}{\mathcal{C}}

\newcommand{\bigO}{\mathcal{O}}

\newcommand{\cS}{\mathcal{S}}
\newcommand{\cA}{\mathcal{A}}
\newcommand{\cB}{\mathcal{B}}

\newcommand{\cM}{\mathcal{M}}

\newcommand{\cG}{\mathcal{G}}

\newcommand{\cK}{\Theta}

\newcommand{\tV}{{\widetilde{V}}}
\newcommand{\tQ}{{\widetilde{Q}}}
\newcommand{\tr}{{\widetilde{r}}}
\newcommand{\hV}{\widehat{V}}
\newcommand{\hbv}{\widehat{\mathbf v}}
\newcommand{\hQ}{\widehat{Q}}
\newcommand{\hmu}{\widehat{\mu}}
\newcommand{\hnu}{\widehat{\nu}}
\newcommand{\hpi}{\widehat{\pi}}
\newcommand{\tpi}{{\widetilde{\pi}}}
\newcommand{\tLambda}{{\widetilde{\Lambda}}}
\newcommand{\hLambda}{{\widehat{\Lambda}}}

\newcommand{\hu}{\widehat{u}}
\newcommand{\tu}{\widetilde{u}}

\newcommand{\hPr}{\widehat{\Pr}}
\newcommand{\hr}{\widehat{r}}
\newcommand{\hbr}{{\widehat{\br}}}
\newcommand{\hProut}{\widehat{\Pr}^{\mathrm{out}}}

\newtheorem{theorem}{Theorem}
\newtheorem{lemma}[theorem]{Lemma}
\newtheorem{corollary}[theorem]{Corollary}

\newtheorem{assumption}[theorem]{Assumption}

\newtheorem{remark}[theorem]{Remark}

\theoremstyle{definition}
\newtheorem{definition}[theorem]{Definition}
\theoremstyle{definition}

\newcommand{\dist}{{\mathrm{dist}}}

\newcommand{\dlin}{d_\mathrm{lin}}

\title{\textbf{A Simple Reward-free Approach to Constrained Reinforcement Learning}}

\author{%
  Sobhan Miryoosefi\\
  Princeton University\\
  \texttt{miryoosefi@cs.princeton.edu}\\
  \and 
  Chi Jin\\
  Princeton University\\
  \texttt{chij@princeton.edu}\\
}
\date{}

\begin{document}

\maketitle

\begin{abstract}

In constrained reinforcement learning (RL), a learning agent seeks to not only optimize the overall reward but also satisfy the additional safety, diversity, or budget constraints. 
Consequently, existing constrained RL solutions require several new algorithmic ingredients that are notably different from standard RL. On the other hand, reward-free RL is independently developed in the unconstrained literature, which learns the transition dynamics without using the reward information, and thus naturally capable of addressing RL with multiple objectives under the common dynamics. This paper bridges reward-free RL and constrained RL. Particularly, we propose a simple meta-algorithm such that given any reward-free RL oracle, the approachability and constrained RL problems can be directly solved with negligible overheads in sample complexity. Utilizing the existing reward-free RL solvers, our framework provides sharp sample complexity results for constrained RL in the tabular MDP setting, matching the best existing results up to a factor of horizon dependence; our framework directly extends to a setting of tabular two-player Markov games, and gives a new result for constrained RL with linear function approximation.

\end{abstract}


\section{Introduction}
\label{sec:intro}

In a wide range of modern reinforcement learning (RL) applications, it is not sufficient for the learning agents to only maximize a scalar reward. More importantly, they must satisfy various \emph{constraints}. For instance, such constraints can be the physical limit of power consumption or torque in motors for robotics tasks \citep{tessler2018reward}; the budget for computation and the frequency of actions for real-time strategy games~\citep{vinyals2019grandmaster}; and the requirement for safety, fuel efficiency and human comfort for autonomous drive~\citep{Hoang2019BPLUC}. In addition, constraints are also crucial in tasks such as dynamic pricing with limited supply~\citep{BZ09,DynPricing-ec12}, scheduling of resources on a computer cluster~\citep{Mao2016RLSystems}, imitation learning~\citep{SyedSchapire2008,ziebart2008maximum,sun2019provably}, as well as reinforcement learning with fairness \citep{jabbari2017fairness}.

These huge demand in practice gives rise to a subfield---constrained RL, which focuses on designing efficient algorithms to find near-optimal policies for RL problems under linear or general convex constraints. Most constrained RL works directly combine the existing techniques such as value iteration and optimism from unconstrained literature, with new techniques specifically designed to deal with linear constraints~\cite{efroni2020exploration,ding2021provably,qiu2020upper} or general convex constraints~\cite{miryoosefi2020,yu2021provably}. The end product is a single new complex algorithm which is tasked to solve all the challenges of learning dynamics, exploration, planning as well as constraints satisfaction simultaneously. Consequently, these algorithms need to be re-analyzed from scratch, and it is highly nontrivial to translate the progress made in the unconstrained RL to the constrained setting. 


On the other hand, reward-free RL---proposed in \cite{jin2020reward}---is a framework for the unconstrained setting, which learns the transition dynamics without using the reward. The framework has two phases: in the exploration phase, the agent first collects trajectories from a Markov decision process (MDP) and learns the dynamics without a pre-specified reward function. After exploration, the agent is tasked with computing near-optimal policies under the MDP for a collection of given reward functions. This framework is particularly suitable when there are multiple reward functions of interest, and has been developed recently to attack various settings including tabular MDPs \cite{jin2020reward, zhang2020task}, linear MDPs \cite{wang2020reward,zanette2020provably}, and tabular Markov games \cite{liu2020sharp}.

\paragraph{Contribution.} In this paper, we propose a simple approach to solve constrained RL problems by bridging the reward-free RL literature and constrained RL literature. Our approach isolates the challenges of constraint satisfaction, and leaves the remaining RL challenges such as learning dynamics and exploration to reward-free RL. This allows us to design a new algorithm which purely focuses on addressing the constraints. Formally, we design a meta-algorithm for RL problems with general convex constraints. 
Our meta-algorithm takes a reward-free RL solver, and can be used to directly solve the approachability problem, as well as the constrained MDP problems using very small amount of samples in addition to what is required for reward-free RL. 


 Our framework enables direct translation of any progress in reward-free RL to constrained RL. Leveraging recent advances in reward-free RL, our meta-algorithm directly implies sample-efficient guarantees of constrained RL in the settings of tabular MDP, linear MDP, as well as tabular two-player Markov games. In particular,
\begin{itemize}
	\item \emph{Tabular setting}: Our work achieves sample complexity of $\tilde{\mathcal{O}}(\min\{d, S\} H^4SA/\epsilon^2)$ for all three tasks of reward-free RL for Vector-valued MDPs (VMDP), approachability, and RL with general convex constraints. Here $d$ is the dimension of VMDP or the number of constraints, $S, A$ are the number of states and actions, $H$ is the horizon, and $\epsilon$ is the error tolerance. It matches the best existing results up to a factor of $H$.
	\item \emph{Linear setting}: Our work provides new sample complexity of $\tilde{\mathcal{O}}(d_{\text{lin}}^3 H^6/\epsilon^2)$ for all three tasks above for linear MDPs. There has been no prior sample-efficient result in this setting.
	\item \emph{Two-player setting}: Our work extends to the setting of tabular two-player vector-valued Markov games and achieves 
  low regret of $\alpha(T)= \bigO(\epsilon/2 + \sqrt{H^2\iota/T})$ at the cost of this $\bigO(\epsilon)$ bias in regret as well as additional samples for preprocessing. 
\end{itemize}

\subsection{Related work}

\begin{table*}[t]
  \renewcommand{\arraystretch}{1.5}
  \caption{\label{table:rate} Sample complexity for algorithms to solve reward-free RL for VMDP (Definition \ref{def:reward-free}), approachability (Definition \ref{def:blackwell-VMDPs}) and CMDP with general convex constraints (Definition \ref{def:constrianed-rl}).\textsuperscript{$1$}
  }
 \centerline{ \begin{tabular}{|c|c|c|c|c|}
    \hline
   & \textbf{Algorithm} & \textbf{Reward-free} & \textbf{Approachability} & \textbf{CMDP} \\ \hline
  \multirow{4}{*}{Tabular} & Wu et al. \cite{wu2020accommodating} & $\tilde{\mathcal{O}}(\min\{d, S\}H^4SA/\epsilon^2)$ & - & - \\\hhline{|~----|}
  & Brantley et al. \cite{miryoosefi2020} & - & - & $\tilde{\mathcal{O}}(d^2H^3S^2A/\epsilon^2)$\\ \hhline{|~----|}
  & Yu et al. \cite{yu2021provably} & - & $\tilde{\mathcal{O}}(\min\{d, S\} H^3SA/\epsilon^2)$ & $\tilde{\mathcal{O}}(\min\{d, S\} H^3SA/\epsilon^2)$\\\hhline{|~----|}
  & \cellcolor{light-gray} This work & $\tilde{\mathcal{O}}(\min\{d, S\} H^4SA/\epsilon^2)$  & $\tilde{\mathcal{O}}(\min\{d, S\} H^4SA/\epsilon^2)$ & $\tilde{\mathcal{O}}(\min\{d, S\} H^4SA/\epsilon^2)$ \\\hline
  Linear &  \cellcolor{light-gray} This work & $\tilde{\mathcal{O}}(d_{\text{lin}}^3 H^6/\epsilon^2)$ & $\tilde{\mathcal{O}}(d_{\text{lin}}^3 H^6/\epsilon^2)$ & $\tilde{\mathcal{O}}(d_{\text{lin}}^3 H^6/\epsilon^2)$\\ \hline
     \end{tabular}
  }
\end{table*}
\newcommand{\customfootnotetext}[2]{{
  \renewcommand{\thefootnote}{#1}
  \footnotetext[0]{#2}}}
\customfootnotetext{$1$}{The presented sample complexities are all under the $L_2$ normalization conditions as studied in this paper. We comment that the results of \citep{wu2020accommodating, miryoosefi2020, yu2021provably} are originally presented under $L_1/L_\infty$ normalization conditions. While the results in \cite{wu2020accommodating} can be directly adapted to our setting as stated in the table, the other two results \cite{miryoosefi2020,yu2021provably} will be no better than the displayed results after adaptation.} 

In this section, we review the related works on three tasks studied in this paper---reward-free RL, approachability, and constrained RL.

\paragraph{Reward-free RL.} Reward-free exploration has been formalized by \cite{jin2020reward} for the tabular setting. Furthermore, \cite{jin2020reward} proposed a algorithm which has sample complexity $\tilde{\mathcal{O}}(\mathrm{poly}(H)S^2A/\epsilon^2)$ outputting $\epsilon$-optimal policy for arbitrary number of reward functions. More recently, \cite{zhang2020task,liu2020sharp} proposes algorithm VI-Zero with sharp sample complexity of $\tilde{\mathcal{O}}(\mathrm{poly}(H)\log(N)SA/\epsilon^2)$ capable of handling $N$ fixed reward functions. \cite{wang2020provably, zanette2020provably} further provide reward-free learning results in the setting of linear function approximation, in particular, \cite{wang2020provably} guarantees to find the near-optimal policies for an arbitrary number of (linear) reward functions within a sample complexity of $\tilde{\mathcal{O}}(\mathrm{poly}(H)\dlin^3/\epsilon^2)$. All results mentioned above are for scalar-valued MDPs. For the vector-valued MDPs (VMDPs), very recent work \cite{wu2020accommodating} designs a reward-free algorithm with sample complexity guarantee $\tilde{\mathcal{O}}(\mathrm{poly}(H)\min\{d,S\}SA/\epsilon^2)$ in the tabular setting. Compared to \cite{wu2020accommodating}, our reward-free algorithms for VMDP is adapted from the VI-Zero algorithm presented in \cite{liu2020sharp}. Our algorithm achieves the same sample complexity as \cite{wu2020accommodating} but allows arbitrary planning algorithms in the planning phase. 



\paragraph{Approachability and Constrained RL} 
Approachability and Constrained RL are two related tasks involving constraints. Inspired by Blackwell approachability \cite{blackwell1956analog}, recent work of \cite{miryoosefi2019} introduces approachability task for VMDPs. However, the proposed algorithm does not have polynomial sample complexity guarantees. More recently, \cite{yu2021provably} gave a new algorithm for approachability for both VMDPs and vector-valued Markov games (VMGs). \cite{yu2021provably} provides regret bounds for the proposed algorithm resulting in sample complexity guarantees of $\tilde{\mathcal{O}}(\poly(H)\min\{d, S\}SA/\epsilon^2)$ for approachability in VMDPs and $\tilde{\mathcal{O}}(\poly(H)\min\{d, S\}SAB/\epsilon^2)$ for approachability in VMGs.

Sample-efficient exploration in constrained reinforcement learning has been recently studied in a recent line of work by \cite{miryoosefi2020,qiu2020upper,efroni2020exploration,ding2021provably,singh2020learning}.  
All these works are also limited to linear constraints except \cite{miryoosefi2020} which extends their approach to general convex constraints achieving sample complexity of $\tilde{\mathcal{O}}(\mathrm{poly}(H)d^2S^2A/\epsilon^2)$ . However, \cite{miryoosefi2020} requires solving a large-scale convex optimization sub-problem. The best result for constrained RL with general convex constraints can be achieved by the approachability-based algorithm in \cite{yu2021provably} obtaining sample complexity of $\tilde{\mathcal{O}}(\mathrm{poly}(H)\min\{d, S\} SA/\epsilon^2)$. Technically, our meta-algorithm is based on the Fenchel's duality, which is similar to \cite{yu2021provably}. In contrast, \cite{yu2021provably} does not use reward-free RL, and is thus  different from our results in terms of algorithmic approaches. Consequently, \cite{yu2021provably} does not reveal the deep connections between reward-free RL and constrained RL, which is one of the main contribution of this paper. In addition, \cite{yu2021provably} does not address the function approximation setting.
 %

	
Finally, we note that among all results mentioned above, only \cite{ding2021provably} has considered models beyond tabular setting in the context of constrained RL. The model studied in \cite{ding2021provably} is known as linear mixture MDPs which is different and incomparable to the linear MDP models considered in this paper.

\section{Preliminaries and problem setup}
\label{sec:prelim}
We consider an episodic \emph{vector-valued Markov decision process} (VMDP) specified by a tuple $\cM = (\cS,\cA,H,\Pr,\br)$, where $\cS$ is the state space, $\cA$ is the action space, $H$ is the length of each episode, $\Pr=\{\Pr_h\}_{h=1}^H$ is the collection of \emph{unknown} transition probabilities with $\Pr_h(s' \mid s,a)$  equal to the probability of transiting to $s'$ after taking action $a$ in state $s$ at the $h^{\text{th}}$ step, and $\br = \{\br_h : \cS \times \cA \rightarrow \cB(1) \}_{h=1}^H$ is a collection of \emph{unknown} $d$-dimensional return functions, where $\cB(r)$ is the $d$-dimensional Euclidean ball of radius $r$ centered at the origin.

\paragraph{Interaction protocol.} In each episode, agent starts at a \emph{fixed} initial state $s_1$. Then, at each step $h \in [H]$, the agent observes the current state $s_h$, takes action $a_h$, receives stochastic sample of the return vector $\br_h(s_h,a_h)$, and it causes the environment to transit to $s_{h+1} \sim \Pr_h(\cdot \mid s_h,a_h)$. We assume that stochastic samples of the return function are also in $\cB(1)$, almost surely. 

\paragraph{Policy and value function.} A policy $\pi$ of an agent is a collection of $H$ functions $\{\pi_h : \cS \rightarrow \Delta(\cA)\}_{h=1}^H$ that map states to distribution over actions. The agent following policy $\pi$, picks action $a_h \sim \pi_h(s_h)$ at the $h^{\text{th}}$ step. We denote ${\bV_h^\pi} : \cS \rightarrow \cB(H)$ as the value function at step $h$ for policy $\pi$, defined as 
\begin{equation*}
	\textstyle{\bV_h^\pi(s):= \E_{\pi}\left[\sum_{h'=h}^H \br_{h'}(s_{h'},a_{h'}) \mid s_h = s \right].}
\end{equation*}
Similarly, we denote $\bQ^\pi_h : \cS \times \cA \rightarrow \cB(H)$ as the $Q$-value function at step $h$ for policy $\pi$, where
\begin{equation*}
	\textstyle{\bQ_h^\pi(s,a):= \E_{\pi}\left[\sum_{h'=h}^H \br_{h'}(s_{h'},a_{h'}) \mid s_h = s, a_h=a \right].}
\end{equation*}
\paragraph{Scalarized MDP.} For a VMDP $\cM$ and $\btheta \in \cB(1)$, we define scalar-valued MDP $\cM_\btheta=(\cS,\cA,H,\Pr,r_\btheta)$, where $r_\btheta = \{ \langle \btheta , \br_h \rangle : \cS \times \cA \rightarrow [-1,1] \}_{h =1}^H$. We denote $V_h^\pi(\cdot ; \btheta): \cS \rightarrow [-H,H]$ as the scalarized value function at step $h$ for policy $\pi$, defined as
\begin{equation*}
	\textstyle{V_h^\pi(s;\btheta):= \E_{\pi}\left[\sum_{h'=h}^H \langle \btheta, \br_{h'}(s_{h'},a_{h'})\rangle \mid s_h = s \right] = \langle \btheta, \bV_h^\pi(s) \rangle.}
\end{equation*}
Similarly, we denote $Q^\pi_h(\cdot;\btheta) : \cS \times \cA \rightarrow [-H,H]$ as the scalarized $Q$-value function at step $h$ for policy $\pi$, where
\begin{equation*}
	\textstyle{Q_h^\pi(s,a;\btheta):= \E_{\pi}\left[\sum_{h'=h}^H \langle \btheta, \br_{h'}(s_{h'},a_{h'}) \rangle \mid s_h = s, a_h=a \right] = \langle \btheta, \bQ_h^\pi(s,a) \rangle.}
\end{equation*}
For a fixed $\btheta \in \mathbb{R}^d$, there exists an optimal policy $\pi^{\star}_\btheta$, maximizing value for all states \citep{puterman2014markov}; i.e., $V^{\pi^{\star}_\btheta}_h(s;\btheta)=\sup_{\pi} V^\pi_h(s;\btheta)$ for all $ s \in \cS$ and $h \in [H]$. We abbreviate $V^{\pi^{\star}_\btheta}(\cdot ; \btheta)$ and $Q^{\pi^{\star}_\btheta}(\cdot ; \btheta)$ as $V^{\star}(\cdot ; \btheta)$ and $Q^{\star}(\cdot ; \btheta)$ respectively.

\subsection{Reward-free exploration (RFE) for VMDPs}
\label{sec:reward-free}
The task of \emph{reward-free exploration} (formalized by \cite{jin2020reward} for tabular MDPs) considers the scenario in which the agents interacts with the environment without guidance of reward information. Later, the reward information is revealed and the agents is required to compute the near-optimal policy. In this section, we introduce its counterpart for VMPDs. Formally, it consists of two phases:

\paragraph{Exploration phase.} In the exploration phase, agent explores the unknown environment without observing any information regarding the return function. Namely, at each episode the agent executes policies to collect samples. The policies can depend on dynamic observations $\{s^k_h,a^k_h\}_{(k,h) \in [K] \times [H]}$ in the past episodes, but not the return vectors.

\paragraph{Planning phase.} In the planning phase, the agent no longer interacts with the environment; however, stochastic samples of the $d$-dimensional return function for the collected episodes is revealed to the agent, i.e. $\{\br^k_h\}_{(k,h) \in [K] \times [H]}$. Based on the episodes collected during the exploration phase, the agent outputs the near-optimal policies of $\cM_\btheta$ given an arbitrary number of vectors $\btheta \in \cB(1)$.

\begin{definition}[Reward-free algorithm for VMDPs]\label{def:reward-free} For any $\epsilon,\delta >0$, after collecting $m_{\textsc{RFE}}(\epsilon,\delta)$ episodes during the exploration phase, with probability at least $1-\delta$, the algorithm satisfies
\begin{equation}
\label{eq:def:reward-free}
	\forall \btheta \in \cB(1): \quad V_1^{\star}(s_1;\btheta)-V^{\pi_\btheta}_1(s_1;\btheta) \leq \epsilon,
\end{equation}
where $\pi_\btheta$ is the output of the planning phase for vector $\btheta$ as input. The function $m_{\textsc{RFE}}$ determines the \emph{sample complexity} of the RFE algorithm. 

\begin{remark} Standard reward-free setup concerns MDPs with scalar reward, and requires the algorithm to find the near-optimal policies for $N$ different prespecified reward functions in the planning phase, where the sample complexity typically scales with $\log N$. This type of results can be adapted into a guarantee in the form of \eqref{eq:def:reward-free} for VMDP by $\epsilon$-covering of $\btheta$ over $\cB(1)$ and a modified concentration arguments (see the proofs of Theorem \ref{thm:tabular} and Theorem \ref{thm:tabular-mg} for more details). 
\end{remark}



\end{definition}
\subsection{Approachability for VMDPs}

In this section we provide the description for the \emph{approachability} task for VMPDs introduced by \cite{miryoosefi2019}. Given a vector-valued Markov decision process and a convex target set $\mathcal{C}$, the goal is to learn a policy whose expected cumulative return vector lies in the target set (akin to Blackwell approachability in single-turn games, \cite{blackwell1956analog}). We consider the agnostic version of this task which is more general since it doesn't need to assume that such policy exists; instead, the agent learns to minimize the Euclidean distance between expected return of the learned policy and the target set.

\begin{definition}[Approachability algorithm for VMDPs]\label{def:blackwell-VMDPs} For any $\epsilon,\delta >0$, after collecting $m_\textsc{APP}(\epsilon,\delta)$ episodes, with probability at least $1-\delta$, the algorithm satisfies
	\begin{equation}
	\label{eq:def:blacwell-approachability}
		\dist(\bV^{\pi^{\mathrm{out}}}_1(s_1),\cC) \leq \min_{\pi}\dist(\bV^\pi_1(s_1),\cC)+\epsilon,
	\end{equation}
	where $\pi^\mathrm{out}$ is the output of the algorithm and $\dist(\bx,\cC)$ is the Euclidean distance between point $\bx$ and set $\cC$. The function $m_{\textsc{APP}}$ determines the \emph{sample complexity} of the algorithm.  
\end{definition}

\subsection{Constrained MDP (CMDP) with general convex constraints}
In this section we describe \emph{constrained Markov decision processes} (CMDPs) introduced by \cite{altman1999constrained}. The goal of this setting is to minimize cost while satisfying some linear constraints over consumption of $d$ resources (resources are akin to $\br$ in our case). Although, the original definition only allows for linear constraints, we consider the more general case of arbitrary convex constraints. 
More formally, consider a VMDP $\cM$, a cost function $c = \{c_h : \cS \times \cA \rightarrow [-1,1] \}_{h=1}^H$, and a convex constraint set $\cC$. The agent goal is to compete against the following benchmark:
\begin{equation*}
	\min_{\pi} C^\pi_1(s_1) \quad \text{s.t.} \quad \bV^\pi_1(s_1) \in \cC,
\end{equation*}
 where $C^\pi_h = \E_{\pi}\left[\sum_{h'=h}^H c_{h'}(s_{h'},a_{h'}) \mid s_h = s \right]$.

\begin{definition}[Algorithm for CMDP]\label{def:constrianed-rl} For any $\epsilon,\delta >0$, after collecting $m_\textsc{CMDP}(\epsilon,\delta)$ episodes, with probability at least $1-\delta$, the algorithm satisfies
	\begin{equation}
	\label{eq:def:constrained-rl}
	\begin{cases}	
		C^{\pi^{\mathrm{out}}}_1(s_1)- {\displaystyle \min_{\pi: \bV^\pi_1(s_1) \in \cC}}C^{\pi}_1(s_1) \leq \epsilon \\
		\dist\big(\bV^{\pi^{\mathrm{out}}}_1(s_1),\cC\big) \leq \epsilon,\\
	\end{cases}
	\end{equation}
	 where $\pi^\mathrm{out}$ is the output of the algorithm. The function $m_{\textsc{CMDP}}$ determines the \emph{sample complexity} of the algorithm.

\end{definition}
As also mentioned in the prior works \citep{miryoosefi2019,yu2021provably}, we formally show in the following theorem that approachability task (Definition~\ref{def:blackwell-VMDPs}) can be considered more general compared to CMDP (Definition~\ref{def:constrianed-rl}); Namely, given any algorithm for the former we can obtain an algorithm for the latter by incurring only extra logarithmic factor and a negligible overhead. The idea is to incorporate cost into the constraint set $\cC$ and perform an (approximate) binary search over the minimum attainable cost. The reduction and the proof can be found in Appendix~\ref{app:prelim}.
\begin{theorem}\label{thm:blackwell-to-constrained}
Given any approachability algorithm (Definition~\ref{def:blackwell-VMDPs}) with sample complexity $m_{\textsc{APP}}$, we can design an algorithm for CMDP (Definition~\ref{def:constrianed-rl}) with sample complexity $m_{\textsc{CMDP}}$, satisfying
\begin{equation*}
	\textstyle{m_{\textsc{CMDP}}(\epsilon,\delta) \leq \bigO\Big( m_{\textsc{APP}}\left(\frac{\epsilon}{6},\frac{\epsilon\delta}{12H}\right) + \frac{H^2\log[dH/\epsilon\delta]}{\epsilon^2}\Big)   \cdot \log\frac{1}{\epsilon}.}
\end{equation*} 
\end{theorem}


\section{Meta-algorithm for VMDPs}
\label{sec:framework}

In this section, equipped with preliminaries discussed in Section~\ref{sec:prelim}, we are ready introduce our main algorithmic framework for VMDPs bridging reward-free RL and approachability. 

Before introducing the algorithm, we explain the intuition behind it. By Fenchel's duality (similar to \cite{yu2021provably}), one can show that 
\begin{equation*}
	\min_{\pi}\dist(\bV^\pi_1(s_1,\cC))=\min_{\pi}\max_{\theta \in \cB(1)} \Big[\langle \btheta, \bV^\pi_1(s_1,\cC) \rangle - \max_{\bx' \in \cC}\langle \btheta,\bx' \rangle \Big].
\end{equation*}
It satisfies the minimax conditions since it's concave in $\btheta$ and convex in $\pi$ (by allowing mixture policies); therefore, minimax theorem \cite{neumann1928theorie} implies that we can equivalently solve
\begin{equation*}
	\max_{\theta \in \cB(1)}\min_{\pi} \Big[\langle \btheta, \bV^\pi_1(s_1,\cC) \rangle - \max_{\bx' \in \cC}\langle \btheta,\bx' \rangle \Big].
\end{equation*}
This max-min form allows us to use general technique of \cite{freund1999adaptive} for solving a max-min by repeatedly playing a no-regret online learning algorithm as the max-player against best-response for the min-player. In particular, for a fixed $\btheta$, minimizing over $\pi$ is equivalent to finding optimal policy for scalarized MDP $\cM_{-\btheta}$. To achieve this, we can utilize a reward-free oracle as in Definition~\ref{def:reward-free}. On the other hand for $\btheta$-player we are able to use online gradient descent \citep{zinkevich2003online}. 


Combining ideas above, we obtain Algorithm~\ref{alg:meta}.

\begin{algorithm}
    \caption{Meta-algorithm for VMDPs}
    \label{alg:meta}
    \begin{algorithmic}[1]
    	\STATE \textbf{Input:} Reward-Free Algorithm $\textsc{Rfe}$ for VMDPs (as in Definiton~\ref{def:reward-free}), Target Set $\cC$ 
        \STATE \textbf{Hyperparameters:} learning rate $\eta^t$
      	\STATE \textbf{Initialize:} run exploration phase of $\textsc{Rfe}$ for $K$ episodes
      	\STATE \textbf{Set:} $\btheta^1 \in \cB(1)$
        \FOR{$t=1,2,\dots,T$}
            \STATE Obtain near optimal policy for $\cM_{-\btheta^t}$:
            \vspace{2mm}
            \begin{center}
            	$\pi^t \leftarrow$ output of planning phase of $\textsc{Rfe}$ for preference vector $-\btheta^t$
            \end{center}
            \vspace{2mm}
            \STATE Estimate $\bV^{\pi^t}_1(s_1)$ using one episode:
            \vspace{2mm}
            \begin{center}
            	Run $\pi^t$ for one episode and let $\hbv^t$ be the sum of vectorial returns            	
            \end{center}
             \vspace{2mm}
            \STATE Apply online gradient ascent update for utility function $u^t(\btheta)=\langle \btheta,\hbv^t \rangle - \max_{\bx \in \cC} \langle \btheta , \bx \rangle$:
            \vspace{2mm}
            \begin{center}	
            $\btheta^{t+1} \leftarrow \Gamma_{\cB(1)}[\btheta^t + \eta^t (\hbv^t-\argmax_{\bx \in \cC} \langle \btheta^t,\bx\rangle)]$
            \end{center}
            \vspace{2mm}
             where $\Gamma_{\cB(1)}$ is the projection into Euclidean unit ball
        \ENDFOR
        \STATE Let $\pi^{\mathrm{out}}$ be uniform mixture of $\{\pi^1,\dots,\pi^T\}$
        \STATE \textbf{Return} $\pi^{\mathrm{out}}$
    \end{algorithmic}
\end{algorithm}

\begin{theorem}
\label{thm:meta}
	There exists an absolute constant $c$, such that for any choice of RFE algorithm (Definition~\ref{def:reward-free}) and for any $\epsilon \in (0,H]$ and $\delta \in (0,1]$, if we choose 
	\begin{equation*}
		T \geq c \big(H^2\iota/\epsilon^2), \quad K \geq  m_\textsc{RFE}(\epsilon/2,\delta/2),  \text{ and} \quad  \eta^t = \sqrt{1/(H^2t)},
	\end{equation*}
	where $\iota=\log(d/\delta)$; then, with probability at least $1-\delta$, Algorithm~\ref{alg:meta} outputs an $\epsilon$-optimal policy for the approachability (Equation~\ref{eq:def:blacwell-approachability}).    
	Therefore, we have 
	$
		m_\textsc{APP}(\epsilon,\delta) \leq \bigO (m_\textsc{RFE}(\epsilon/2,\delta/2) +H^2\iota/\epsilon^2).
	$
\end{theorem}

Theorem~\ref{thm:meta} shows that given any reward-free algorithm, Algorithm~\ref{alg:meta} can solve the approachability task with negligible overhead. The proof for Theorem~\ref{thm:meta} is provided in Appendix~\ref{app:framework}. Equipped with this theorem, since we have already shown the connection between approachability and constrained RL in Theorem~\ref{thm:blackwell-to-constrained}, any results for RFE can be directly translated to results for constrained RL.


\section{Tabular VMDPs}
\label{sec:tabular}
In this section, we consider tabular VMDPs; namely, we assume that $|\cS|\leq S$ and $|\cA| \leq A$. Utilizing prior work on tabular setting, describe our choice of reward-free algorithm.

In the exploration phase, we use VI-Zero proposed by \cite{liu2020sharp} which is displayed in Algorithm~\ref{alg:tabular}. Intuitively, the value function $\tQ_h(s,a)$ computed in the algorithm measures the level of uncertainty that agent may suffer if it takes action $a$ at state $s$ in step $h$. It incentivize the greedy policy to visit underexplored states improving our empirical estimate $\hPr$.

In the planning phase, given $\btheta \in \cB(1)$ as input we can use any planning algorithm (such as value iteration) for $\widehat{\cM}_\btheta = (\cS,\cA,H,\hProut,\langle \btheta,\hbr \rangle)$ where $\hbr$ is empirical estimate of $\br$ using collected samples $\{\br^k_h\}$. 

The following theorem state theoretical guarantees of the reward-free algorithm for tabular VMDPs. Proof of Theorem~\ref{thm:tabular} and more details can be found in Appendix~\ref{app:tabular}. 

\begin{algorithm}
    \caption{VI-Zero: Exploration Phase}
    \label{alg:tabular}
    \begin{algorithmic}[1]
    	\STATE \textbf{Hyperparameters:} Bonus $\beta_t$.
    	\STATE \textbf{Initialize:} 
    	for all $(s,a,h)\in \cS \times \cA \times [H]$: $\tQ_h(s,a)\leftarrow H$ and $N_h(s,a) \leftarrow 0$,    	
    	\STATE \hspace{14.5mm} for all $(s,a,h,s') \in \cS \times \cA \times [H] \times \cS$: $N_h(s,a,s') \leftarrow 0$,
		\STATE \hspace{14.5mm} $\Delta \leftarrow 0$.
    	\FOR{episode $k=1,2,\dots,K$}
    		\FOR{step $h=H,H-1,\dots,1$}
    			\FOR{state-action pair $(s,a) \in \cS \times \cA$}
    				\STATE $t \leftarrow N_h(s,a)$.
    				\IF{$t>0$}
    					\STATE $\tQ_h(s,a) \leftarrow \min\{[\hPr_h\tV_{h+1}](s,a)+\beta_t,H\}$.
    				\ENDIF
    			\ENDFOR
    			\FOR{state $s \in \cS$}
    				\STATE $\tV_h(s) \leftarrow \max_{a \in \cA}\tQ_h(s,a)$ and  $\pi_h(s) \leftarrow \argmax_{a \in \cA}\tQ_h(s,a)$
    			\ENDFOR
    		\ENDFOR
    		\IF{ $\tV(s_1) \leq \Delta$}
    			\STATE $\Delta \leftarrow \tV(s_1)$ and $\hProut \leftarrow \hPr_h$
    		\ENDIF
    		\FOR{step $h=1,2,\dots,H$}
    			\STATE Take action $a_h \leftarrow \pi_h(s_h)$ and observe next state $s_{h+1}$
    			\STATE Update $N_h(s_h,a_h) \leftarrow N_h(s_h,a_h) + 1$ and $N_h(s_h,a_h,s_{h+1}) \leftarrow N_h(s_h,a_h,s_{h+1}) + 1$
    			\STATE $\hPr_h(\cdot \mid s_h,a_h) \leftarrow N_h(s_h,a_h,\cdot) / N_h(s_h,a_h)$
    		\ENDFOR
    	\ENDFOR
    	\STATE \textbf{Return} $\hProut$
    \end{algorithmic} 
\end{algorithm}


\begin{theorem}
\label{thm:tabular}
There exists a reward-free algorithm for tabular VMDPs and a right choice of hyperparameters that satisfies Definition~\ref{def:reward-free} with sample complexity
$
	m_{\textsc{RFE}}(\epsilon,\delta)	 \leq \bigO(\min\{d,S\}H^4SA\iota/\epsilon^2+H^3S^2A\iota^2/\epsilon),
$
where $\iota=\log[dSAH/(\epsilon\delta)]$.
\end{theorem}

The theorem shows that we can achieve sample complexity of $\tilde{\bigO}(\min\{d,S\}H^4SA/\epsilon^2)$ for the reward-free exploration (Definition~\ref{def:reward-free}) in the tabular setting, matching the best result in \cite{wu2020accommodating}. The following corollary immediately follows from Theorem~\ref{thm:tabular} along with  Theorem~\ref{thm:meta} and \ref{thm:blackwell-to-constrained-alg}.


\begin{corollary}
\label{cor:tabular} In the tabular setting we can have an algorithm for  
	approachability  (Definition~\ref{def:blackwell-VMDPs}) with  
	$
		m_{\textsc{APP}}(\epsilon,\delta)	 \leq \bigO(\min\{d,S\}H^4SA\iota/\epsilon^2+H^3S^2A\iota^2/\epsilon)
	$	
	, and an algorithm for CMDP (Definition~\ref{def:constrianed-rl}) with 
	$
		m_{\textsc{CMDP}}(\epsilon,\delta)	 \leq \bigO(\min\{d,S\}H^4SA\iota^2/\epsilon^2+H^3S^2A\iota^3/\epsilon)
	$.
\end{corollary}

The corollary shows that we can achieve sample complexity of $\tilde{\bigO}(\min\{d,S\}H^4SA/\epsilon^2)$ for approachability (Definition~\ref{def:blackwell-VMDPs}) and CMDP with general convex constraints (Definition~\ref{def:constrianed-rl}) in the tabular setting, matching the best result in \cite{yu2021provably} up to a single factor of $H$. \footnote{This $H$ factor difference is due the Bernstein-type bonus used in \cite{yu2021provably}, which can not be adapted to the reward-free setting.} Therefore, our framework while being modular enabling direct translation of reward-free RL to constrained RL, achieves sharp sample complexity guarantees. We comment that due to reward-free nature of our approach unlike \cite{yu2021provably}, we can no longer provide regret guarantees.

\section{Linear function approximation: Linear VMDPs}
\label{sec:linear}

In this section we consider the setting of linear function approximation and allow $\cS$ and $\cA$ to be infinitely large. We assume that agent has access to a feature map $\bphi: \cS \times \cA \rightarrow \mathbb{R}^{\dlin}$ and the return function and transitions are linear functions of the feature map. We formally define the linear VMDPs in Assumption~\ref{asm:linear} which adapts the definition of linear MDPs \citep{jin2020provably} for VMDPS; namely, they coincide for the case of $d=1$.

\begin{assumption}[Linear VMDP]
\label{asm:linear}
A VMDP $\cM = (\cS,\cA,H,\Pr,\br)$ is said to be a linear VMDP with a feature map $\bphi :\cS \times \cA \rightarrow \mathbb{R}^{\dlin}$, if for any $h \in [H]$:
\begin{enumerate}
	\item There exists $\dlin$ unknown (signed) measures $\bmu_h=\{\mu_h^{(1)},\dots,\mu_h^{(\dlin)}\}$ over $\cS$ such that for any $(s,a) \in \cS \times \cA$ we have $\Pr_h(\cdot\mid s,a)=\langle\bmu(\cdot),\bphi(s,a)\rangle$.
	\item There exists an unknown matrix $W_h \in \mathbb{R}^{d \times \dlin}$ such that for any $(s,a) \in \cS \times \cA$ we have $\br_h(s,a)=W_h\bphi(s,a)$.
\end{enumerate}
\end{assumption}
Similar to \citet{jin2020provably}, we assume that $\norm{\bphi(s,a)} \leq 1$ for all $(s,a) \in \cS \times \cA$, $\norm{\bmu_h(\cS)}\leq \sqrt{\dlin}$ for all $h \in [H]$, and $\norm{W_h} \leq \sqrt{\dlin}$ for all $h \in [H]$. 

\cite{wang2020provably} has recently proposed a sample-efficient algorithm for reward-free exploration in linear MDPs. Utilizing that algorithm and tailoring it for our setting, we can obtain the following theoretical guarantee. The algorithm and the proof can be found in Appendix~\ref{app:linear}.
\begin{theorem}
\label{thm:linear}
There exists a reward-free algorithm for linear VMDPs (as in Assumption~\ref{asm:linear}) and a right choice of hyperparameters that satisfies Definition~\ref{def:reward-free} with sample complexity
$
	m_{\textsc{RFE}}(\epsilon,\delta)	 \leq \bigO(\dlin^3H^6\iota^2/\epsilon^2),
$
where $\iota=\log[\dlin d H/(\epsilon\delta)]$.
\end{theorem}

The theorem provides a new sample complexity result of $\tilde{\bigO}(\dlin^3H^6/\epsilon^2)$ for the reward-free exploration (Definition~\ref{def:reward-free}) in the linear setting (Assumption~\ref{asm:linear}). The following corollary immediately follows from Theorem~\ref{thm:linear} along with  Theorem~\ref{thm:meta} and \ref{thm:blackwell-to-constrained-alg}.


\begin{corollary}
\label{cor:linear} In the linear setting (Assumption~\ref{asm:linear}) we can have
 an approachability algorithm (Definition~\ref{def:blackwell-VMDPs}) satisfying 
	$
		m_{\textsc{APP}}(\epsilon,\delta)	 \leq \bigO(\dlin^3H^6\iota^2/\epsilon^2)
	$ and 	
	an algorithm for CMDP (Definition~\ref{def:constrianed-rl}) with satisfying
	$
		m_{\textsc{CMDP}}(\epsilon,\delta)	 \leq \bigO(\dlin^3H^6\iota^3/\epsilon^2)	
	$.
\end{corollary}

The corollary provides a new sample complexity result of $\tilde{\bigO}(\dlin^3H^6/\epsilon^2)$ for both  approachability (Definition~\ref{def:blackwell-VMDPs}) and CMDP (Definition~\ref{def:constrianed-rl}) in the linear setting (Assumption~\ref{asm:linear}). 


\section{Vector-valued Markov games}
\label{sec:mg}

\subsection{Model and preliminaries}

Similar to Section~\ref{sec:prelim}, we consider an episodic \emph{vector-valued Markov game} (VMG) specified by a tuple $\cG = (\cS,\cA,\cB,H,\Pr,\br)$, where $\cA$ and $\cB$ are the action spaces for the min-player and max-player, respectively. The $d$-dimensional return function $\br$ and the transition probabilities $\Pr$, now depend on the current state and the action of both players. 

\paragraph{Interaction protocol.} In each episode, we start at a \emph{fixed} initial state $s_1$. Then, at each step $h \in [H]$, both players observe the current state $s_h$, take their own actions $a_h \in \cA$ and $b_h \in \cB$ simultaneously, observe stochastic sample of the return vector $\br_h(s_h,a_h,b_h)$ along with their opponent's action, and it causes the environment to transit to $s_{h+1} \sim \Pr_h(\cdot \mid s_h,a_h,b_h)$. We assume that stochastic samples of the return function are also in $\cB(1)$, almost surely.

\paragraph{Policy and value function.} A policy $\mu$ of the min-player is a collection of $H$ functions $\{\mu_h:\cS \rightarrow \Delta(\cA)\}_{h=1}^H$. Similarly, a policy $\nu$ of the max-player is a collection of $H$ functions $\{\nu_h:\cS \rightarrow \Delta(\cB)\}_{h=1}^H$. If the players are following $\mu$ and $\nu$, we have $a_h \sim \mu(\cdot | s)$ and $b_h \sim \nu(\cdot | s)$ at the $h^\text{th}$ step. We use $\bV^{\mu,\nu}_h:\cS \rightarrow \cB(H)$ and $\bQ^{\mu,\nu}_h:\cS \times \cA \times \cB \rightarrow \cB(H)$ to denote the value function and Q-value function at step $h$ under policies $\mu$ and $\nu$.

\paragraph{Scalarized markov game and Nash equilibrium.} For a VMG $\cG$ and $\btheta \in \cB(1)$, we define scalar-valued Markov game $\cG_\btheta=(\cS,\cA,H,\Pr,r_\btheta)$, where $r_\btheta = \{ \langle \btheta , \br_h \rangle : \cS \times \cA \times \cB \rightarrow [-1,1] \}_{h =1}^H$. We use $V^{\mu,\nu}_h(\cdot;\btheta)$ and $Q^{\mu,\nu}_h(\cdot,\cdot,\cdot;\btheta)$ to denote value function and Q-value function of $\cG_\btheta$, respectively. Note that we have $V^{\mu,\nu}_h(s;\btheta)=\langle \btheta , \bV^{\mu,\nu}_h(s)\rangle$ and $Q^{\mu,\nu}_h(s,a,b;\btheta)=\langle \btheta , \bQ^{\mu,\nu}_h(s,a,b) \rangle$. 

For any policy of the min-player $\mu$, there exists a \emph{best-response} policy $\nu_\dagger(\mu)$ of the max-player; i.e. $V^{\mu,\nu_\dagger(\mu)}_h(s;\btheta) = \max_{\nu} V^{\mu,\nu}_h(s;\btheta)$ for all $(s,h) \in \cS \times [H]$. We use $V^{\mu,\dagger}$ to denote $V^{\mu,\nu_\dagger(\mu)}$.  Similarly, we can define $\mu_\dagger(\nu)$ and $V^{\dagger,\nu}$. We further know \citep{filar2012competitive} that there exist policies $(\mu^\star,\nu^\star)$, known as \emph{Nash equilibrium}, satisfying the following equation for all $(s,h)\in \cS \times [H]$:
\begin{equation*}
	\min_{\mu}\max_{\nu}V^{\mu,\nu}_h(s;\btheta)=V^{\mu^\star,\dagger}_h(s;\btheta)=V^{\mu^\star,\nu^\star}_h(s;\btheta)= V^{\dagger,\nu^\star}_h(s;\btheta)=\max_{\nu}\min_{\mu} V^{\mu,\nu}(s;\btheta)
\end{equation*}
	In words, it means that no player can gain anything by changing her own policy. We abbreviate $V^{\mu^\star,\nu^\star}_h$ and $Q^{\mu^\star,\nu^\star}_h$ as $V^\star_h$ and $Q^\star_h$.

\subsubsection{Reward-free exploration (RFE) for VMGs}

Similar to Section~\ref{sec:reward-free}, we can define RFE algorithm for VMGs. Similarly, it consists of two phases. In the exploration phase the it explores the environment without guidance of return function. Later, in the planning phase, given any $\btheta \in \cB(1)$, it requires to output near optimal Nash equilibrium for $\cG_\btheta$.
%

\begin{definition}[RFE algorithm for VMGs]	
\label{def:reward-free-mg} 
For any $\epsilon,\delta >0$, after collecting $m_{\textsc{RFE}}(\epsilon,\delta)$ episodes during the exploration phase, with probability at least $1-\delta$, the algorithm for all $\btheta \in \cB(1)$, satisfies
\begin{equation*}
	 V^{\mu_\btheta,\dagger}_1(s_1;\btheta) - V^{\dagger,\nu_{\btheta}}_1(s_1;\btheta) = [V^{\mu_\btheta,\dagger}_1(s_1;\btheta)-V^\star_1(s_1;\btheta)]+[V^\star_1(s_1;\btheta)-V^{\dagger,\nu_{\btheta}}_1(s_1;\btheta)] \leq \epsilon
\end{equation*}
where $(\mu_\btheta,\nu_\btheta)$ is the output of the planning phase for vector $\btheta$ as input. The function $m_{\textsc{RFE}}$ determines the \emph{sample complexity} of the RFE algorithm. 
\end{definition}

\subsubsection{Blackwell approachability for VMGs}
We assume we are given a VMG $\cG$ and a target set $\cC$. The goal of the min-player is for the return vector to lie in the set $\cC$ while max-player wants the opposite. For the two-player vector-valued games it can be easily shown that the minimax theorem does no longer hold (see Section 2.1 of \cite{abernethy2011blackwell}). Namely, if for every policy of the max-player we have a response such that the return is in the set, we cannot hope to find a single policy for the min-player so that for every policy of the max-player the return vector lie in the set. However, approaching the set on average is possible.

\begin{definition}[Blackwell approachability] \label{def:blackwell-mg} We say the min-player is approaching the target $\cC$ with rate $\alpha(T)$, if for arbitrary sequence of max-player polices $\nu^1,\dots,\nu^T$, we have
\begin{equation*} 
	\textstyle{\dist(\frac{1}{T}\sum_{t=1}^T\bV^{\mu^t,\nu^t}_1(s_1),\cC) \leq \max_{\nu}\min_{\mu}\dist(\bV^{\mu,\nu}_1(s_1),\cC)+\alpha(T).}
\end{equation*}
\end{definition}

\subsection{Meta-algorithm for VMGs}
Similar to Section~\ref{sec:framework}, we introduce our main algorithmic framework for VMGs bridging reward-free algorithm and Blackwell approachability in VMGs. The pseudo-code is displayed in Algorithm~\ref{alg:meta-mg} and the theoretical guarantees are provided in Theorem~\ref{thm:meta-mg}. The proof can be found in Appendix~\ref{app:mg}.
\begin{algorithm}
    \caption{Meta-algorithm for VMGs}
    \label{alg:meta-mg}
    \begin{algorithmic}[1]
    	\STATE \textbf{Input:} Reward-Free Algorithm $\textsc{Rfe}$ for VMG (as in Definition~\ref{def:reward-free-mg}), Target Set $\cC$ 
        \STATE \textbf{Hyperparameters:} learning rate $\eta^t$
      	\STATE \textbf{Initialize:} run exploration phase of $\textsc{Rfe}$ for $K$ episodes
      	\STATE \textbf{Set:} $\btheta^1 \in \cB(1)$
        \FOR{$t=1,2,\dots,T$}
            \STATE Obtain near optimal Nash equilibrium for $\cG_{\btheta^t}$:
            \vspace{2mm}
            \begin{center}
            	$(\mu^t,\omega^t) \leftarrow$ output of planning phase of $\textsc{RFE}$ for the vector $\btheta^t$ as input 
            \end{center}
            \vspace{2mm}
            \STATE Play $\mu^t$ for one episode:
            \vspace{2mm}
            \begin{center}
            	Play $\mu^t$ against max-player playing arbitrary policy $\nu^t$ for one episode\\ and let $\hbv^t$ be the sum of vectorial returns          	
            \end{center}
             \vspace{2mm}
            \STATE Apply online gradient ascent update for utility function $u^t(\btheta)=\langle \btheta,\hbv^t \rangle - \max_{\bx \in \cC} \langle \btheta , \bx \rangle$:
            \vspace{2mm}
            \begin{center}	
            $\btheta^{t+1} \leftarrow \Gamma_{\cB(1)}[\btheta^t + \eta^t (\hbv^t-\argmax_{\bx \in \cC} \langle \btheta^t,\bx\rangle)]$
            \end{center}
            \vspace{2mm}
             where $\Gamma_{\cB(1)}$ is the projection into Euclidean unit ball
        \ENDFOR
    \end{algorithmic}
\end{algorithm}

\begin{theorem}
\label{thm:meta-mg}
	For any choice of RFE algorithm (Definition~\ref{def:reward-free-mg}) and for any $\epsilon \in (0,H]$ and $\delta \in (0,1]$, if we choose 
	$
		 K =  m_\textsc{RFE}(\epsilon/2,\delta/2)$ and $\eta^t = \sqrt{1/H^2t}
	$
	; then, with probability at least $1-\delta$, the min-player in Algorithm~\ref{alg:meta-mg}, satisfies Definition~\ref{def:blackwell-mg} with rate $\alpha(T)= \bigO(\epsilon/2 + \sqrt{H^2\iota/T})$ where $\iota=\log(d/\delta)$. Therefore to obtain $\epsilon$-optimality, the total sample complexity scales with $\bigO(m_\textsc{RFE}(\epsilon/2,\delta/2) +H^2\iota/\epsilon^2)$.
\end{theorem}

%

\subsection{Tabular VMGs}
In this section, we consider tabular VMDPs; namely, we assume that $|\cS|\leq S$, $|\cA| \leq A$, and $|\cB| \leq B$. Similar to Section~\ref{sec:tabular}, by utilizing VI-Zero \citep{yu2021provably} we can have the following theoretical guarantees. The algorithm and the proof can be found in Appendix~\ref{app:mg}.

\begin{theorem}
\label{thm:tabular-mg}
There exists a reward-free algorithm for tabular VMGs and a right choice of hyperparameters that satisfies Definition~\ref{def:reward-free-mg} with sample complexity
$
	m_{\textsc{RFE}}(\epsilon,\delta)	 \leq \bigO(\min\{d,S\}H^4SAB\iota/\epsilon^2+H^3S^2AB\iota^2/\epsilon),
$
where $\iota=\log[dSABH/(\epsilon\delta)]$.
\end{theorem}
The theorem provides a new sample complexity result of $\tilde{\bigO}(\min\{d,S\}H^4SAB\iota/\epsilon^2)$ for reward-free exploration in VMGs (Definition~\ref{def:reward-free-mg}). It immediately follows from Theorem~\ref{thm:tabular-mg} and Theorem~\ref{thm:meta-mg} that we can achieve total sample complexity of $\tilde{\bigO}(\min\{d,S\}H^4SAB\iota/\epsilon^2)$ for Blackwell approachability in VMGs (Definition~\ref{def:blackwell-mg}). Our rate for $\alpha(T)$ scales with $\tilde{\bigO}(\sqrt{\mathrm{poly}(H)/T})$ while the results in \cite{yu2021provably} has the rate of $\alpha(T)$ scaling with $\tilde{\bigO}(\sqrt{\mathrm{poly}(H)\min\{d,S\}SA/T})$. However, we require initial phase of self-play for $K=\bigO(m_\textsc{RFE})$ episodes which is not needed by \cite{yu2021provably}.




\section{Conclusion}
\label{sec:conclusion}


This paper provides a meta algorithm that takes a reward-free RL solver, and convert it to an algorithm for solving constrained RL problems. Our framework enables the direct translation of any progress in reward-free RL to constrained RL setting. Utilizing existing reward-free solvers, our framework provides sharp sample complexity results for constrained RL in tabular setting (matching best existing results up to factor of horizon dependence), new results for the linear function approximation setting. Our framework further extends to tabular two-player vector-valued Markov games for solving Blackwell approachability problem. 


\bibliographystyle{apalike}
\bibliography{ref.bib}

\begin{thebibliography}{}

\bibitem[Abernethy et~al., 2011]{abernethy2011blackwell}
Abernethy, J., Bartlett, P.~L., and Hazan, E. (2011).
\newblock Blackwell approachability and no-regret learning are equivalent.
\newblock In {\em Proceedings of the 24th Annual Conference on Learning
  Theory}, pages 27--46. JMLR Workshop and Conference Proceedings.

\bibitem[Altman, 1999]{altman1999constrained}
Altman, E. (1999).
\newblock {\em Constrained Markov decision processes}, volume~7.
\newblock CRC Press.

\bibitem[Azar et~al., 2017]{azar2017minimax}
Azar, M.~G., Osband, I., and Munos, R. (2017).
\newblock Minimax regret bounds for reinforcement learning.
\newblock In {\em International Conference on Machine Learning}, pages
  263--272. PMLR.

\bibitem[Babaioff et~al., 2015]{DynPricing-ec12}
Babaioff, M., Dughmi, S., Kleinberg, R.~D., and Slivkins, A. (2015).
\newblock Dynamic pricing with limited supply.
\newblock {\em TEAC}, 3(1):4.
\newblock Special issue for \emph{13th ACM EC}, 2012.

\bibitem[Besbes and Zeevi, 2009]{BZ09}
Besbes, O. and Zeevi, A. (2009).
\newblock Dynamic pricing without knowing the demand function: Risk bounds and
  near-optimal algorithms.
\newblock {\em Operations Research}, 57(6):1407--1420.

\bibitem[Blackwell et~al., 1956]{blackwell1956analog}
Blackwell, D. et~al. (1956).
\newblock An analog of the minimax theorem for vector payoffs.
\newblock {\em Pacific Journal of Mathematics}, 6(1):1--8.

\bibitem[Brantley et~al., 2020]{miryoosefi2020}
Brantley, K., Dudik, M., Lykouris, T., Miryoosefi, S., Simchowitz, M.,
  Slivkins, A., and Sun, W. (2020).
\newblock Constrained episodic reinforcement learning in concave-convex and
  knapsack settings.
\newblock In {\em Advances in Neural Information Processing Systems},
  volume~33, pages 16315--16326. Curran Associates, Inc.

\bibitem[Ding et~al., 2021]{ding2021provably}
Ding, D., Wei, X., Yang, Z., Wang, Z., and Jovanovic, M. (2021).
\newblock Provably efficient safe exploration via primal-dual policy
  optimization.
\newblock In Banerjee, A. and Fukumizu, K., editors, {\em Proceedings of The
  24th International Conference on Artificial Intelligence and Statistics},
  volume 130 of {\em Proceedings of Machine Learning Research}, pages
  3304--3312. PMLR.

\bibitem[Efroni et~al., 2020]{efroni2020exploration}
Efroni, Y., Mannor, S., and Pirotta, M. (2020).
\newblock Exploration-exploitation in constrained mdps.
\newblock {\em arXiv preprint arXiv:2003.02189}.

\bibitem[Filar and Vrieze, 2012]{filar2012competitive}
Filar, J. and Vrieze, K. (2012).
\newblock {\em Competitive Markov decision processes}.
\newblock Springer Science \& Business Media.

\bibitem[Freund and Schapire, 1999]{freund1999adaptive}
Freund, Y. and Schapire, R.~E. (1999).
\newblock Adaptive game playing using multiplicative weights.
\newblock {\em Games and Economic Behavior}, 29(1-2):79--103.

\bibitem[Jabbari et~al., 2017]{jabbari2017fairness}
Jabbari, S., Joseph, M., Kearns, M., Morgenstern, J., and Roth, A. (2017).
\newblock Fairness in reinforcement learning.
\newblock In {\em International Conference on Machine Learning}, pages
  1617--1626. PMLR.

\bibitem[Jin et~al., 2020a]{jin2020reward}
Jin, C., Krishnamurthy, A., Simchowitz, M., and Yu, T. (2020a).
\newblock Reward-free exploration for reinforcement learning.
\newblock In {\em International Conference on Machine Learning}, pages
  4870--4879. PMLR.

\bibitem[Jin et~al., 2019]{jin2019short}
Jin, C., Netrapalli, P., Ge, R., Kakade, S.~M., and Jordan, M.~I. (2019).
\newblock A short note on concentration inequalities for random vectors with
  subgaussian norm.
\newblock {\em arXiv preprint arXiv:1902.03736}.

\bibitem[Jin et~al., 2020b]{jin2020provably}
Jin, C., Yang, Z., Wang, Z., and Jordan, M.~I. (2020b).
\newblock Provably efficient reinforcement learning with linear function
  approximation.
\newblock In {\em Conference on Learning Theory}, pages 2137--2143. PMLR.

\bibitem[Le et~al., 2019]{Hoang2019BPLUC}
Le, H.~M., Voloshin, C., and Yue, Y. (2019).
\newblock Batch policy learning under constraints.
\newblock {\em CoRR}, abs/1903.08738.

\bibitem[Liu et~al., 2020]{liu2020sharp}
Liu, Q., Yu, T., Bai, Y., and Jin, C. (2020).
\newblock A sharp analysis of model-based reinforcement learning with
  self-play.
\newblock {\em arXiv preprint arXiv:2010.01604}.

\bibitem[Mao et~al., 2016]{Mao2016RLSystems}
Mao, H., Alizadeh, M., Menache, I., and Kandula, S. (2016).
\newblock Resource management with deep reinforcement learning.
\newblock In {\em Proceedings of the 15th ACM Workshop on Hot Topics in
  Networks}, page 50–56, New York, NY, USA. Association for Computing
  Machinery.

\bibitem[Miryoosefi et~al., 2019]{miryoosefi2019}
Miryoosefi, S., Brantley, K., Daume~III, H., Dudik, M., and Schapire, R.~E.
  (2019).
\newblock Reinforcement learning with convex constraints.
\newblock In {\em Advances in Neural Information Processing Systems},
  volume~32, pages 14093--14102. Curran Associates, Inc.

\bibitem[Neumann, 1928]{neumann1928theorie}
Neumann, J.~v. (1928).
\newblock Zur theorie der gesellschaftsspiele.
\newblock {\em Mathematische annalen}, 100(1):295--320.

\bibitem[Puterman, 2014]{puterman2014markov}
Puterman, M.~L. (2014).
\newblock {\em Markov decision processes: discrete stochastic dynamic
  programming}.
\newblock John Wiley \& Sons.

\bibitem[Qiu et~al., 2020]{qiu2020upper}
Qiu, S., Wei, X., Yang, Z., Ye, J., and Wang, Z. (2020).
\newblock Upper confidence primal-dual optimization: Stochastically constrained
  markov decision processes with adversarial losses and unknown transitions.
\newblock {\em arXiv preprint arXiv:2003.00660}.

\bibitem[Rockafellar, 2015]{rockafellar2015convex}
Rockafellar, R.~T. (2015).
\newblock {\em Convex analysis}.
\newblock Princeton university press.

\bibitem[Singh et~al., 2020]{singh2020learning}
Singh, R., Gupta, A., and Shroff, N.~B. (2020).
\newblock Learning in markov decision processes under constraints.
\newblock {\em arXiv preprint arXiv:2002.12435}.

\bibitem[Sun et~al., 2019]{sun2019provably}
Sun, W., Vemula, A., Boots, B., and Bagnell, J.~A. (2019).
\newblock Provably efficient imitation learning from observation alone.
\newblock {\em arXiv preprint arXiv:1905.10948}.

\bibitem[Syed and Schapire, 2007]{SyedSchapire2008}
Syed, U. and Schapire, R.~E. (2007).
\newblock A game-theoretic approach to apprenticeship learning.
\newblock In {\em Proceedings of the 20th International Conference on Neural
  Information Processing Systems}, NIPS’07, page 1449–1456, Red Hook, NY,
  USA. Curran Associates Inc.

\bibitem[Tessler et~al., 2019]{tessler2018reward}
Tessler, C., Mankowitz, D.~J., and Mannor, S. (2019).
\newblock Reward constrained policy optimization.
\newblock In {\em International Conference on Learning Representations}.

\bibitem[Vinyals et~al., 2019]{vinyals2019grandmaster}
Vinyals, O., Babuschkin, I., Czarnecki, W.~M., Mathieu, M., Dudzik, A., Chung,
  J., Choi, D.~H., Powell, R., Ewalds, T., Georgiev, P., et~al. (2019).
\newblock Grandmaster level in starcraft ii using multi-agent reinforcement
  learning.
\newblock {\em Nature}, 575(7782):350--354.

\bibitem[Wang et~al., 2020a]{wang2020reward}
Wang, R., Du, S.~S., Yang, L.~F., and Salakhutdinov, R. (2020a).
\newblock On reward-free reinforcement learning with linear function
  approximation.
\newblock {\em arXiv preprint arXiv:2006.11274}.

\bibitem[Wang et~al., 2020b]{wang2020provably}
Wang, R., Salakhutdinov, R., and Yang, L.~F. (2020b).
\newblock Provably efficient reinforcement learning with general value function
  approximation.
\newblock {\em arXiv preprint arXiv:2005.10804}.

\bibitem[Wu et~al., 2020]{wu2020accommodating}
Wu, J., Braverman, V., and Yang, L.~F. (2020).
\newblock Accommodating picky customers: Regret bound and exploration
  complexity for multi-objective reinforcement learning.
\newblock {\em arXiv preprint arXiv:2011.13034}.

\bibitem[Yu et~al., 2021]{yu2021provably}
Yu, T., Tian, Y., Zhang, J., and Sra, S. (2021).
\newblock Provably efficient algorithms for multi-objective competitive rl.
\newblock {\em arXiv preprint arXiv:2102.03192}.

\bibitem[Zanette et~al., 2020]{zanette2020provably}
Zanette, A., Lazaric, A., Kochenderfer, M.~J., and Brunskill, E. (2020).
\newblock Provably efficient reward-agnostic navigation with linear value
  iteration.
\newblock {\em arXiv preprint arXiv:2008.07737}.

\bibitem[Zhang et~al., 2020]{zhang2020task}
Zhang, X., Singla, A., et~al. (2020).
\newblock Task-agnostic exploration in reinforcement learning.
\newblock {\em arXiv preprint arXiv:2006.09497}.

\bibitem[Ziebart et~al., 2008]{ziebart2008maximum}
Ziebart, B.~D., Maas, A.~L., Bagnell, J.~A., and Dey, A.~K. (2008).
\newblock Maximum entropy inverse reinforcement learning.
\newblock In {\em Aaai}, volume~8, pages 1433--1438. Chicago, IL, USA.

\bibitem[Zinkevich, 2003]{zinkevich2003online}
Zinkevich, M. (2003).
\newblock Online convex programming and generalized infinitesimal gradient
  ascent.
\newblock In {\em Proceedings of the 20th international conference on machine
  learning (icml-03)}, pages 928--936.

\end{thebibliography}


\appendix
\newpage
\section{Proof for Section~\ref{sec:prelim}}
\label{app:prelim}
In this section we provide proofs and missing details for Section~\ref{sec:prelim}.
\subsection{Proof of Theorem~\ref{thm:blackwell-to-constrained}}
\label{app:blackwell-to-constrained}
Consider the following algorithm which is performing an approximate version of binary search on the optimal cost. We use $\oplus$ to denote vector concatenation.

\begin{algorithm}
    \caption{Solving Constrained RL Using Approachability}
    \begin{algorithmic}[1]
    \label{alg:blackwell-to-constrained}
    	\STATE \textbf{Input:} approachability algorithm $\mathrm{APP}$
        \STATE \textbf{Hyperparameters:} $\epsilon' > 0$
      	\STATE \textbf{Initialize:} $L \leftarrow 0$ and $R \leftarrow H$ 
      	\STATE Define the augmented VMDP model
      	\begin{center}
      		$\overline{\br}_h(s,a)=\br_h(s,a) \oplus c_h(s,a) \quad \forall h \in [H]$\\
      		$\overline{\cM}=\{\cS,\cA,H,\Pr,\overline{\br}\}$
      	\end{center} 
      	\FOR{iteration $t=1,2,\dots,T$}
      	\STATE Set $\mathrm{mid} = (R+L)/2$
      	\STATE Define the target set for approachability
      	\begin{center}
      		$\overline{\cC}^t=\{ \bx \oplus y \mid \bx \in \cC, y \leq \mathrm{mid}\}$\\
      	\end{center}
      	\STATE $\pi^t \leftarrow$  output of $\mathrm{APP}$ algorithm for the model $\overline{\cM}$ with target set $\overline{\cC}^t$ using $K_{\mathrm{APP}}$ episodes.
      	\STATE $\overline{\bv}^t \leftarrow $ estimate $\overline{\bV}^{\pi^t}_1(s_1)$ using $K_{\mathrm{est}}$ episodes, where $\overline{\bV}$ is the value function for $\overline{\cM}$.
      	\IF{$\dist(\overline{\bv}^t,\cC) \leq \epsilon'$}
      		\STATE $R \leftarrow \mathrm{mid}$
      	\ELSE
      		\STATE $L \leftarrow \mathrm{mid}$
      	\ENDIF
      	\ENDFOR
     	\STATE \textbf{Return} $\pi^T$
    \end{algorithmic}

\end{algorithm}

\begin{theorem}
\label{thm:blackwell-to-constrained-alg}
For any choice of approachability algorithm (as in Definiton~\ref{def:blackwell-VMDPs}) and for any $\epsilon,\delta >0$, if we choose
\begin{equation*}
	T = \bigO\big[\log(H/\epsilon)\big], \quad K_{\mathrm{APP}}=m_{\mathrm{APP}}(\epsilon,\epsilon\delta/(2H)), \quad K_{\mathrm{est}}=\bigO \big[\frac{H^2\log(dH/\epsilon\delta)}{\epsilon^2}\big], \quad \epsilon' = \bigO(\epsilon)
\end{equation*} 
then, with probability at least $1-\delta$, Algorithm~\ref{alg:blackwell-to-constrained} satisfies
\begin{equation*}
\begin{cases}	
	C^{\pi^{T}}_1(s_1)- {\displaystyle \min_{\pi: \bV^\pi_1(s_1) \in \cC}}C^{\pi}_1(s_1) \leq \bigO(\epsilon), \\
	\dist\big(\bV^{\pi^{T}}_1(s_1),\cC\big) \leq \bigO(\epsilon).\\
\end{cases}
\end{equation*}
\end{theorem}

\begin{proof}[Proof of Theorem~\ref{thm:blackwell-to-constrained-alg}]
By definition~\ref{def:blackwell-VMDPs}, Lemma~\ref{lem:concentration-l2vector}, and union bound; with probability at least $1-\delta$, we have for all $t \in [T]$
\begin{equation}
\label{eq:blackwell-to-constrained-alg-1}
\begin{aligned}
	&\norm{\overline{\bv}^t-\overline{\bV}^{\pi^t}_1(s_1)}\leq \epsilon,\\
	&\dist(\overline{\bV}^{\pi^{t}}_1(s_1),\cC) \leq \min_{\pi}\dist(\overline{\bV}^\pi_1(s_1),\cC)+\epsilon.
\end{aligned}
\end{equation}
We use $L^t$, $R^t$, and $\mathrm{mid}^t$ to denote values of $L$, $R$, and $\mathrm{mid}$ during $t^\text{th}$ iteration.
By choice of $T$ we  have 
\begin{equation}
\label{eq:blackwell-to-constrained-alg-2}
	R^T-L^T \leq \epsilon.
\end{equation}
Define $c^* = \min_{\pi: \bV^\pi_1(s_1) \in \cC}C^{\pi}_1(s_1)$ and let $\pi^* = \argmin_{\pi: \bV^\pi_1(s_1) \in \cC}C^{\pi}_1(s_1)$.
Let's consider these cases
\begin{itemize}
	\item Case $\mathrm{mid \geq c^*}$: It's easy to see that $ \min_{\pi}\dist(\overline{\bV}^\pi_1(s_1),\cC) = 0$, therefore by second inequality in Equation~\ref{eq:blackwell-to-constrained-alg-1} we have 
	\begin{equation*}
		\dist(\overline{\bV}^{\pi^{t}}_1(s_1),\cC) \leq \epsilon.
	\end{equation*}
	Since distance function is $1$-Lipschitz with respect to Euclidean norm, by first inequality in Equation~\ref{eq:blackwell-to-constrained-alg-1}, we have 
	\begin{equation*}
		\dist(\overline{\bv}^t,\cC) \leq \epsilon + \epsilon = 2\epsilon
	\end{equation*}
	\item Case $\mathrm{mid} \leq c^*-3\epsilon$:  It's easy to see that $ \min_{\pi}\dist(\overline{\bV}^\pi_1(s_1),\cC) \geq 3\epsilon$, therefore by definition of minimum we have 
	\begin{equation*}
		\dist(\overline{\bV}^{\pi^{t}}_1(s_1),\cC) \geq 3\epsilon.
	\end{equation*}	
	Since distance function is $1$-Lipschitz with respect to Euclidean norm, by first inequality in Equation~\ref{eq:blackwell-to-constrained-alg-1}, we have 
	\begin{equation*}
		\dist(\overline{\bv}^t,\cC) \geq 3\epsilon - \epsilon = 2\epsilon.
	\end{equation*}
\end{itemize}
What we showed above implies that if we set $\epsilon' = 2\epsilon$, in all iterations $t \in [T]$ we have
\begin{equation*}
	L^t \leq c^*, \quad  R^t \geq c^*-3\epsilon.
\end{equation*}
Combining with Equation~\ref{eq:blackwell-to-constrained-alg-2}, we get 
\begin{equation*}
	c^*-4\epsilon \leq L^T \leq \mathrm{mid}^T \leq R^T \leq c^*+\epsilon 
\end{equation*}
Therefore we have, 
\begin{equation*}
\begin{aligned}
	&\max\{C^{\pi^{T}}_1(s_1)- \mathrm{mid}^T, \dist\big(\bV^{\pi^{T}}_1(s_1),\cC\big)\}  \\
	&\leq \dist(\overline{\bV}^{\pi^{T}}_1(s_1),\cC)\\
	&\leq \min_{\pi}\dist(\overline{\bV}^\pi_1(s_1),\cC) + \epsilon\\
	&\leq \dist(\overline{\bV}^{\pi^*}_1(s_1),\cC)+ \epsilon\\
	&\leq \max\{c^*-\mathrm{mid}^T,0\}+\epsilon \\
	&\leq c^*-(c^*-4\epsilon)+\epsilon = 5 \epsilon
\end{aligned}
\end{equation*}
It implies
\begin{equation*}
	\begin{cases}
		\dist\big(\bV^{\pi^{T}}_1(s_1),\cC\big) \leq 5\epsilon\\
		C^{\pi^{T}}_1(s_1) \leq 5\epsilon + \mathrm{mid}^T \leq c^* + 6\epsilon\\
	\end{cases}
\end{equation*}
Rescaling $\epsilon$ to $\epsilon/6$ completes the proof.
\end{proof}

\begin{proof}[Proof of Theorem~\ref{thm:blackwell-to-constrained}]
	Using Theorem~\ref{thm:blackwell-to-constrained-alg} the claim follows immediately: total sample complexity of Algorithm~\ref{alg:blackwell-to-constrained} is 
	\begin{equation*}
		T(K_{\mathrm{APP}}+K_{\mathrm{est}}) \leq \log(1/\epsilon) \cdot \bigO\Big( m_{\textsc{APP}}(\epsilon,\epsilon\delta/H) + \frac{H^2\log[d/\epsilon\delta]}{\epsilon^2}\Big). 
	\end{equation*}
\end{proof}

\section{Proof for Section~\ref{sec:framework}}
\label{app:framework}
In this section we provide proofs and missing details for Section~\ref{sec:framework}.
\subsection{Fenchel duality}
Consider a convex and closed function $f: \mathrm{dom}(f) \rightarrow \mathbb{R}$. We define the dual function $f^*$, called \emph{Fenchel conjugate}, as 
\begin{equation*}
	f^*(\btheta) = \max_{\bx \in \mathrm{dom}(f)} \Big[ \langle \btheta, \bx  \rangle - f(\bx)\Big]. 
\end{equation*}  
If function $f$ is $1$-Lipschitz and $\mathrm{dom}(f)=\cB(H)$; then, the conjugate function $f^*$ is $H$-Lipschitz with $\mathrm{dom}(f^*)=\cB(1)$ (Corollary 13.3.3 in \cite{rockafellar2015convex}). Therefore, Fenchel daulity implies
\begin{equation*}
	f(\bx)=\max_{\btheta \in \cB(1)} \Big[\langle \btheta , \bx \rangle - f^*(\btheta)\Big].
\end{equation*}
In particular, for closed, convex, and $1$-Lipschitz function $f$ defined as $$\begin{cases} f:\cB(H)\rightarrow \mathbb{R} \\
	f(\bx)=\dist(\bx,\cC)\\ \end{cases}$$ we have
\begin{equation*}
	f^*(\btheta) = \max_{\bx \in \cC} \langle \btheta,\bx \rangle.
\end{equation*}
It's easy to verify that $\partial f^*(\btheta) = \argmax_{\bx \in \cC} \langle \btheta, \bx \rangle$ is a subgradient of $f^*$ at $\btheta$. Fenchel duality implies that
\begin{equation}
\label{eq:dist-fenchel}
	\dist(x,\cC)= \max_{\btheta \in \cB(1)} \Big [ \langle \btheta, \bx \rangle - \max_{\bx' \in \cC} \langle \btheta, \bx' \rangle \Big ].
\end{equation}


\subsection{Online Convex Optimization (OCO)}

We will be using the guarantee of online gradient ascent algorithm \citep{zinkevich2003online} in the proof. Therefore, we briefly review the framework of online convex optimization. We can imagine an online game between the leaner and the environment: The learner is given a decision set $\cK$; at time $t=1,2,\dots,T$, the leaner makes a decision $\btheta^t \in \cK$, the environment reveals a concave utility function $u^t: \cK \rightarrow \mathbb{R}$, and the learner gains utility $u^t(\btheta^t)$. The learner's goal is to minimize \emph{regret} defined as 
\begin{equation*}
	\mathrm{Regret}_T \triangleq \max_{\btheta \in \cK} \big[ \sum_{t=1}^T u^t(\btheta)\big] - \big[\sum_{t=1}^T u^t(\btheta^t)\big].
\end{equation*}
An OCO algorithm is \emph{no-regret} if $\mathrm{Regret}_T=o(T)$, meaning its average utility approaches to best in hindsight. The \emph{online gradient ascent} (OGA) is an example of such algorithm (Algorithm~\ref{alg:oga}). In Theorem~\ref{thm:oga} we formally state the theoretical guarantee of this algorithm.

\begin{algorithm}[h]
	\caption{Online gradient ascent (OGA)}
	\label{alg:oga}
	\begin{algorithmic}[1]
 	 \STATE \textbf{input}: projection operator $\Gamma_{\cK}$ where $\Gamma_{\cK}(\btheta)=\mathrm{argmin}_{\btheta \in \cK}\norm {\btheta - \btheta'}$
 	 \STATE \textbf{init}: $\btheta^1$ arbitrarily
 	 \STATE \textbf{parameters}: step size $\eta_t$
 	 \FOR {$t=1$ \TO $T$}
 	 	\STATE observe concave utility function $u^t:\cK \rightarrow \mathbb{R}$
 	 	\STATE ${\mathbf \btheta}^{t+1}= \Gamma_{\cK}\big(\btheta^t + \eta_t \partial u^t(\btheta^t)\big)$ \COMMENT{where $\partial u^t(\btheta^t)$ is a subgradient of $u^t$ at $\btheta^t$}
 	 	\ENDFOR
	\end{algorithmic}
\end{algorithm}

\begin{theorem}[\cite{zinkevich2003online}]
\label{thm:oga}	
 Assume that for any $\btheta,\btheta' \in \cK$ we have $\norm{\btheta - \btheta'}  \leq D$ and $u^1,\dots,u^T$ are concave and $G$-Lipschitz.
By setting $\eta_t = \frac{D}{G\sqrt{t}}$, Algorithm~\ref{alg:oga} satisfies
\[
\mathrm{Regret}_T\leq \bigO(DG\sqrt{T}).
\]
\end{theorem}

\subsection{Proof of Theorem~\ref{thm:meta}}
We use the following choice for parameters:
\begin{equation}
\label{eq:app_framework_choice_param}
	K \geq  m_{\mathrm{RFE}}(\epsilon/2, \delta/2),  \quad T \geq c  \cdot (H^2\iota/\epsilon^2).
\end{equation}
We denote $\bv^t:=\bV^{\pi^t}_1(s_1)$ and start with the following lemma.

\begin{lemma}
\label{lem:event0}
	Define even $E_0$ to be:
\begin{equation*}
	\begin{cases}
		\norm{\frac{1}{T}\sum_{t=1}^T\bv^t-{\hbv^t}} \leq \bigO(\sqrt{H^2\iota/T}), \\
		\bV^*_1(s_1;-\btheta^t) \leq \bV^{\pi^t}_1(s_1;-\btheta^t) + \epsilon/2 \quad \forall t \in [T].
	\end{cases}
\end{equation*}
where $\iota=\log(d/\delta)$. We have $\Pr(E_0) \geq 1-\delta$.
\end{lemma}

\begin{proof}[Proof of Lemma~\ref{lem:event0}]
We show that each claim holds with probability at least $1-\delta/2$; applying a union bound completes the proof.
\paragraph{First claim.} Let $\mathcal{F}_t$ be the filtration capturing all the randomness in the algorithm before iteration $t$. We have $\E[\hbv^t \mid \mathcal{F}_t]=\bv^t$ and we also know that $\norm{\hbv^t} \leq H$ almost surely. By applying Lemma~\ref{lem:concentration-l2vector}, with probability at least $1-\delta$ we have
\begin{equation*}
	\norm{\frac{1}{T}\sum_{t=1}^T\bv^t-{\hbv^t}} \leq \bigO(\sqrt{H^2\log[d/\delta]/T}),
\end{equation*} 
which completes the proof.
\paragraph{Second claim.} Choice of parameters in Equation~\ref{eq:app_framework_choice_param} along with Definition~\ref{def:reward-free} immediately implies that with probability at least $1-\delta/2$ we have 
\begin{equation*}
	\bV^*_1(s_1;-\btheta^t) \leq \bV^{\pi^t}_1(s_1;-\btheta^t) + \epsilon/2 \quad \forall t \in [T].
\end{equation*}
Note that in Algorithm~\ref{alg:meta}, $\pi^t$ is the output of the planning phase of the RFE algorithm for the vector $-\btheta^t$ as input.
\end{proof}

The following lemma states that if $\alpha = \min_{\pi} \dist(\bV^\pi_1(s_1),\cC) \geq 0$ is the closest achievable distance to target set $\cC$, then any halfspace containing $\cC$ is reachable up to error $\alpha$.
\begin{lemma}
\label{lem:delta-approachability}
	For any $\btheta \in \cB(1)$, we have
	\begin{equation*}
		\min_{\bx \in \cC} \langle \btheta,\bx \rangle \leq \min_{\pi} \dist(\bV^\pi_1(s_1),\cC)+\bV^*_1(s_1;\btheta).
	\end{equation*}
\end{lemma}

\begin{proof}[Proof of Lemma~\ref{lem:delta-approachability}]
	Let $\overline{\pi} = \argmin_\pi \dist(\bV^\pi_1(s_1),\cC)$ and define $\overline{\bv} = \bV^{\overline{\pi}}_1(s_1)$. Let $\tilde{\bv} = \Gamma_{\cC}(\overline{\bv})$ be the orthogonal projection of $\overline{\bv}$ into $\cC$. We have 
	\begin{equation*}
	\begin{aligned}
		\bV^*_1(s_1;\btheta) &\geq \bV^{\overline{\pi}}_1(s_1;\btheta) \\
		&= \langle \btheta , \overline{\bv} \rangle\\
		&= \langle \btheta , \overline{\bv}-\tilde{\bv} \rangle + \langle \btheta, \tilde{\bv} \rangle\\
		&\geq -\norm{\overline{\bv}-\tilde{\bv}}+\min_{\bx \in \cC} \langle \btheta,\bx\rangle \\
		&\geq -\min_{\pi} \dist(\bV^\pi_1(s_1),\cC)+\min_{\bx \in \cC} \langle \btheta,\bx\rangle
	\end{aligned}
	\end{equation*}
\end{proof}

Now we are ready to proceed with proof of Theorem~\ref{thm:meta}.

\begin{proof}[Proof of Theorem~\ref{thm:meta}]
With probability at least $1-\delta$ event $E_0$ holds and we have
\begin{equation*}
\begin{aligned}
\dist(\bV^{\pi^{\mathrm{out}}}_1(s_1),\cC) &=
\dist\Big(\frac{1}{T}\sum_{t=1}^T\bv^t,\cC\Big)\\
&\overset{(i)}{=} \max_{\btheta \in \cB(1)} \Big[\langle \btheta , \frac{1}{T}\sum_{t=1}^T\bv^t)\rangle - \max_{\bx \in \cC} \langle \btheta, \bx \rangle \Big]\\
&=\max_{\btheta \in \cB(1)} \Big[\frac{1}{T}\sum_{t=1}^T\big(\langle \btheta, {\hbv^t} \rangle - \max_{\bx \in \cC}\langle \btheta, \bx  \rangle) + \langle \btheta, \frac{1}{T}\sum_{t=1}^T \bv^t-{\hbv^t} \rangle \Big]\\
&\overset{(ii)}{\leq}\max_{\btheta \in \cB(1)} \Big[\frac{1}{T}\sum_{t=1}^T\big(\langle \btheta, {\hbv^t} \rangle - \max_{\bx \in \cC}\langle \btheta, \bx  \rangle)\Big] + \bigO(\sqrt{H^2\iota/T})\\
&\overset{(iii)}{\leq}\frac{1}{T}\sum_{t=1}^T\big(\langle \btheta^t, {\hbv^t} \rangle - \max_{\bx \in \cC}\langle \btheta^t, \bx  \rangle) + \bigO(\sqrt{H^2/ T})+ \bigO(\sqrt{H^2\iota/T})\\
&\overset{(iv)}{\leq} \min_{\pi} \dist(\bV^\pi_1(s_1),\cC) + \frac{1}{T}\sum_{t=1}^T\big(\langle \btheta^t, {\hbv^t} \rangle + \bV^*_1(s_1;-\btheta^t)\big) +\bigO(\sqrt{H^2\iota/T})\\
&\overset{(v)}{\leq} \min_{\pi} \dist(\bV^\pi_1(s_1),\cC) + \epsilon/2+ \frac{1}{T}\sum_{t=1}^T\big(\langle \btheta^t, {\hbv^t} \rangle + \bV^{\pi^t}_1(s_1;-\btheta^t)\big) +\bigO(\sqrt{H^2\iota/T})\\
&= \min_{\pi} \dist(\bV^\pi_1(s_1),\cC) + \epsilon/2+ \frac{1}{T}\sum_{t=1}^T\langle \btheta^t, {\hbv^t} -\bv^t \rangle +\bigO(\sqrt{H^2\iota/T})\\
&\overset{(vi)}{\leq}\min_{\pi} \dist(\bV^\pi_1(s_1),\cC) + \epsilon/2 +\bigO(\sqrt{H^2\iota/T})\\
&\overset{(vii)}{\leq} \min_{\pi} \dist(\bV^\pi_1(s_1),\cC) + \epsilon
\end{aligned}	
\end{equation*}
where $(i)$ is by Equation~\ref{eq:dist-fenchel}, $(ii)$ is by first inequality in event $E_0$ together with Cauchy-Schwarz, $(iii)$ is by guarantee of OGA in Theorem~\ref{thm:oga}, $(iv)$ is by Lemma~\ref{lem:delta-approachability}, $(v)$ is by second inequality in event $E_0$, $(vi)$ is by first inequality in event $E_0$ together with Cauchy-Schwarz, and finally $(vii)$ is by setting $T \geq c \big(H^2\iota/\epsilon^2\big)$ for large enough constant $c$, completing the proof.
\end{proof}

\section{Proof for Section~\ref{sec:tabular}}
\label{app:tabular}
In this section we provide proofs and missing details for Section~\ref{sec:tabular}.
\subsection{Proof of Theorem~\ref{thm:tabular}}

Let $\hPr^k$ and $\hbr^k$ be our empirical estimates of the transition and the return vectors at the beginning of the $k^\mathrm{th}$ episode in Algorithm~\ref{alg:tabular} and define $\widehat{\cM}^k=(\cS,\cA,H,\hPr^k,\hbr^k)$. We use $N^k_h(s,a)$ to denote the number of times we have visited state-action $(s,a)$ in step $h$ before $k^\mathrm{th}$ episode in Algorithm~\ref{alg:tabular}. We use superscript $k$ to denote variable corresponding to episode $k$; in particular, $(s^k_1,a^k_1,\dots,s^k_H,a^k_H)$ is the trajectory we have visited in the $k^\text{th}$ episode.

For any $\btheta \in \cB(1)$, let $\widehat{\cM}^k_\btheta$ be the scalarized MDP using vector $\btheta$ (defined in Section~\ref{sec:prelim}). We use $\hV^k(\cdot;\btheta)$, $\hQ^k(\cdot,\cdot;\btheta)$, and $\hpi^k_\btheta=\hpi^k(\cdot;\btheta)$ to denote the optimal value function, optimal Q-value function, and optimal policy of $\widehat{\cM}^k_\btheta$ respectively. Therefore, we have
\begin{equation}
\label{eq:bellman-app-tabular}
\begin{aligned}
	&\hQ^k_h(s,a;\btheta)=[\hPr^k_h\hV^k_{h+1}](s,a;\btheta)+\hr^k_h(s,a;\btheta),\\
	&\hV^k_h(s;\btheta)=\max_{a \in \cA}\hQ^k_h(s,a;\btheta),\\
	&\hpi^k_h(s;\btheta)=\argmax_{a \in \cA}\hQ^k_h(s,a;\btheta).
\end{aligned}
\end{equation}

\begin{theorem}[restatement of Theorem~\ref{thm:tabular}]
\label{restate:thm:tabular}
There exist absolute constants $c_\beta$ and $c_K$, such that for any $\epsilon \in (0,H]$, $\delta  \in (0,1]$, if we choose bonus $\beta_t = c_\beta\big(\sqrt{\min\{d,S\}H^2\iota/t}+H^2S\iota/t\big)$ where $\iota=\log[dSAKH/\delta]$, and run the exploration phase (Algorithm~\ref{alg:tabular}) for $K \geq c_K\big(\min\{d,S\}H^4SA\iota'/\epsilon^2+H^3S^2A(\iota')^2/\epsilon\Big)$ episodes where $\iota'=\log[dSAH/(\epsilon\delta)]$, then with probability at least $1-\delta$, the algorithm satisfies
\begin{equation*}
	\forall \btheta \in \cB(1): \quad V_1^{\star}(s_1;\btheta)-V^{\pi_\btheta}_1(s_1;\btheta) \leq \epsilon,
\end{equation*}
where $\pi_\btheta$ is the output of the any planning algorithm (e.g., value iteration) for the MDP $\widehat{\cM}^\mathrm{out}_\btheta$. Therefore, we have 
\begin{equation*}
m_{\textsc{RFE}}(\epsilon,\delta)	 \leq \bigO\Big(\frac{\min\{d,S\}H^4SA\iota'}{\epsilon^2}+\frac{H^3S^2A(\iota')^2}{\epsilon}\Big).
\end{equation*}
\end{theorem}

The bonus for episode $k$ can be written as
\begin{equation}
\label{eq:bonus-tabular}
	\beta^k_h(s,a) = c_\beta  \Big(\sqrt{\frac{\min\{d,S\}H^2\iota}{\max\{N^k_h(s,a),1\}}}+\frac{H^2S\iota}{\max\{N^k_h(s,a),1\}} \Big),
\end{equation}
where $\iota = \log[dSAKH/\delta]$ and $c_\beta$ is some large absolute constant.

We begin with the following lemma showing that the value function for a fixed $\pi$ and also the optimal value function is $H$-Lipschitz with respect to $\btheta$.
\begin{lemma}
\label{lem:lipschitz-V}
For all $(s,h) \in \cS \times [H]$, for all policies $\pi$, and for any two vectors $\btheta,\btheta' \in \cB(1)$, we have
\begin{equation*}
\begin{aligned}
	&|V_h^{\star}(s;\btheta)-V_h^{\star}(s;\btheta')| \leq (H-h+1)\norm{\btheta-\btheta'}\\
	&|V_h^{\pi}(s;\btheta)-V_h^{\pi}(s;\btheta')| \leq (H-h+1)\norm{\btheta-\btheta'} \quad 
\end{aligned}
\end{equation*}
\end{lemma}
\begin{proof}[Proof of Lemma~\ref{lem:lipschitz-V}]
We prove each claim separately.
\paragraph{First claim.}
	We prove the lemma by backward induction on $h$. For $h=H+1$ we have $V_h^{\star}(s;\btheta)=V_h^{\star}(s;\btheta')=0$ and the inequality holds. Now assume that $|V_{h+1}^{\star}(s;\btheta)-V_{h+1}^{\star}(s;\btheta')| \leq (H-h)\norm{\btheta-\btheta'}$ holds, we want to show that the claim also holds for $h$. We have
	\begin{equation*}
	\begin{aligned}
		|V_h^{\star}(s;\btheta)-V_h^{\star}(s;\btheta')| &= |\max_{a \in \cA}Q^{\star}_h(s,a;\btheta) - \max_{a' \in \cA}Q^{\star}(s,a';\btheta')|\\
		&\leq \max_{a \in \cA}|Q^{\star}_h(s,a;\btheta)-Q^{\star}_h(s,a;\btheta')|\\
		&= \max_{a \in \cA}|\langle \btheta-\btheta',\br_h(s,a)\rangle + \sum_{s' \in \cS}\Pr(s' \mid s,a)(V^{\star}_{h+1}(s';\btheta)-V^{\star}_{h+1}(s';\btheta'))\\
		&\leq \max_{a \in \cA}||\langle \btheta-\btheta',\br_h(s,a)\rangle|+\max_{a \in \cA}|\sum_{s' \in \cS}\Pr(s' \mid s,a)(V^{\star}_{h+1}(s';\btheta)-V^{\star}_{h+1}(s';\btheta'))|\\
		&\leq \norm{\btheta-\btheta'} + (H-h)\norm{\btheta-\btheta'}\\
		&=(H-h+1)\norm{\btheta-\btheta'}.\\
	\end{aligned}
	\end{equation*}
	It completes the proof of the lemma.
\paragraph{Second claim.}
The second claim is much easier to prove, since we have
\begin{equation*}
\begin{aligned}
	|V_h^{\pi}(s;\btheta)-V_h^{\pi}(s;\btheta')| &= \Big|\E_\pi\big[ \sum_{h'=1}^H\langle \btheta-\btheta',\br_h(s_h',a_h') \rangle \big]\Big|\\
	&\leq \E_\pi\big[\sum_{h'=1}^H |\langle \btheta-\btheta',\br_h(s_h',a_h')\rangle|\big]\\
	&\leq E_\pi\big[\sum_{h'=1}^H \norm{ \btheta-\btheta'}\big]\\
	&=(H-h+1)\norm{ \btheta-\btheta'}\\
\end{aligned}
\end{equation*}
where the first inequality uses Jensen, and second inequality uses Cauchy-Schwarz.
\end{proof}

\begin{lemma}
\label{lem:event1}
Let $c$ be some large absolute constant such that $2c+12c^2 \leq c_\beta$. Define event $E_1$ to be: for all $(s,a,s',h) \in \cS \times \cA \times \cS \times [H]$, $k \in [K]$, and $ \btheta \in \cB(1)$,
\begin{equation}
	\begin{cases}
		|[(\hPr_h^k-\Pr_h)V^{\star}_{h+1}](s,a;\btheta)|&\leq {c}\sqrt{\frac{\min\{d,S\}H^2\iota}{\max\{N^k_h(s,a),1\}}},\\
		|(\hr^k_h-r_h)(s,a;\btheta)|&\leq {c}\sqrt{\frac{\iota}{\max\{N^k_h(s,a),1\}}},\\
		|(\hPr_h^k-\Pr_h)(s' \mid s,a)|&\leq {c}\Big(\sqrt{\frac{\hPr_h^k(s'\mid s,a)\iota}{\max\{N^k_h(s,a),1\}}}+\frac{\iota}{\max\{N^k_h(s,a),1\}}\Big),
	\end{cases}
\end{equation}
where $\iota = \log[dSAKH/\delta]$. We have $\Pr(E_1) \geq 1-\delta$.
\end{lemma}

\begin{proof}[Proof of Lemma~\ref{lem:event1}]
The proof is by applying concentration and covering arguments together with union bounds. The following shows that each claim holds with probability at least $1-\delta$; rescaling $\delta$ to $\delta/3$ and applying a union bound completes the proof.
\paragraph{First claim:} For a fixed $(s,a,k,h,\btheta) \in \cS \times \cA \times [K] \times [H] \times \cB(1)$, using Azuma-Hoeffding inequality, with probability at least $1-\delta'$ we have
\begin{equation*}
	|[(\hPr_h^k-\Pr_h)V^{\star}_{h+1}](s,a;\btheta)|\leq \bigO\Big(\sqrt{\frac{H^2\log(1/\delta')}{N^k_h(s,a)}}\Big).
\end{equation*}
Now consider an $\epsilon'$-covering $\cB_{\epsilon'}$ for the unit Euclidean ball $\cB(1)$ with $\log |\cB_{\epsilon'}| \leq \bigO(d \log(1/\epsilon'))$. For any $\btheta \in \cB(1)$, there exists $\btheta' \in \cB_{\epsilon'}$ satisfying $\norm{\btheta - \btheta} \leq \epsilon'$. The concentration inequality above along with a union bound implies that with probability at least $1-\delta$ for any $(s,a,k,h,\btheta')\in \cS \times \cA \times [K] \times [H] \times \cB_{\epsilon'}$ we have 
\begin{equation*}
	|[(\hPr_h^k-\Pr_h)V^{\star}_{h+1}](s,a;\btheta')|\leq \bigO\Big(\sqrt{\frac{dH^2}{N^k_h(s,a)}\log(\frac{SAKH}{\epsilon'\delta})}\Big).
\end{equation*}
Now consider an arbitrary $(s,a,k,h,\btheta) \in \cS \times \cA \times [K] \times [H] \times \cB(1)$. Let $\btheta' \in \cB_{\epsilon'}$ be such that $\norm{\btheta - \btheta'} \leq \epsilon'$; we have
\begin{equation*}
\begin{aligned}
	&|[(\hPr_h^k-\Pr_h)V^{\star}_{h+1}](s,a;\btheta)|\\ &\overset{(i)}{\leq} |[\hPr_h^k(V^{\star}_{h+1}(\cdot;\btheta)-V^{\star}_{h+1}(\cdot;\btheta'))](s,a)]| + |[(\hPr_h^k-\Pr_h)V^{\star}_{h+1}](s,a;\btheta')|\\ &+ |[\Pr_h(V^{\star}_{h+1}(\cdot;\btheta')-V^{\star}_{h+1}(\cdot;\btheta))](s,a)]| \\
	&\overset{(ii)}{\leq} 2H\norm{\btheta - \btheta'} + \bigO\Big(\sqrt{\frac{dH^2}{N^k_h(s,a)}\log(\frac{SAKH}{\epsilon'\delta})}\Big)\\
	&\leq 2H\epsilon' + \bigO\Big(\sqrt{\frac{dH^2}{N^k_h(s,a)}\log(\frac{SAKH}{\epsilon'\delta})}\Big),
\end{aligned}
\end{equation*}
where $(i)$ is by adding and subtracting the term $[(\hPr_h^k-\Pr_h)V^{\star}_{h+1}](s,a;\btheta')$ along with triangle inequality, and $(ii)$ is by Lemma~\ref{lem:lipschitz-V}. Setting $\epsilon' = \frac{1}{H N_h^k(s,a)} \geq \frac{1}{HK}$ results in
\begin{equation*}
	[(\hPr_h^k-\Pr_h)V^{\star}_{h+1}](s,a;\btheta)| \leq \bigO\Big(\sqrt{\frac{dH^2}{N^k_h(s,a)}\log(\frac{SAKH}{\delta})}\Big).
\end{equation*}
On the other hand, consider an $\epsilon'$-cover $\mathcal{V}_{\epsilon'}$ for the $\ell_\infty$ ball of radius $H$ in dimension $S$, i.e. $\{ \bv \in \mathbb{R}^S \mid \norm{\bv}_\infty \leq H\}$. For a fixed $(s,a,k,h,\bv) \in \cS \times \cA \times [K] \times [H] \times \mathcal{V}_{\epsilon'}$, using Azuma-Hoeffding inequality, with probability at least $1-\delta'$ we have
\begin{equation*}
	|[(\hPr_h^k-\Pr_h)\bv](s,a)|\leq \bigO\Big(\sqrt{\frac{H^2\log(1/\delta')}{N^k_h(s,a)}}\Big).
\end{equation*}
Note that $|\mathcal{V}_{\epsilon'}| \leq (3H/\epsilon')^d$, therefore by putting $\delta' = \delta/(SAKH|\mathcal{V}_{\epsilon'}|)$ we get for all $(s,a,k,h,\bv) \in \cS \times \cA \times [K] \times [H] \times \mathcal{V}_{\epsilon'}$
\begin{equation*}
	|[(\hPr_h^k-\Pr_h)\bv](s,a)|\leq \bigO\Big(\sqrt{\frac{SH^2\log(SAKH/(\epsilon'\delta))}{N^k_h(s,a)}}\Big).
\end{equation*}
Now consider an arbitrary  $(s,a,k,h,\btheta) \in \cS \times \cA \times [K] \times [H] \times \cB(1)$, and let $\bv \in \mathcal{V}_{\epsilon'}$ be such that $\norm{V^{\star}_{h+1}(\cdot ; \btheta) - \bv}_\infty \leq \epsilon'$. We have
\begin{equation*}
\begin{aligned}
	|[(\hPr_h^k-\Pr_h)V^{\star}_{h+1}](s,a;\btheta)| &\leq |[\hPr_h^k(V^{\star}_{h+1}(\cdot;\btheta)-\bv)](s,a)|+ |[(\hPr_h^k-\Pr_h)\bv](s,a)|\\
	&\hspace{5mm} + |[\Pr_h(V^{\star}_{h+1}(\cdot;\btheta)-\bv)](s,a)|\\
	&\leq 2\epsilon' + \bigO\Big(\sqrt{\frac{SH^2\log(SAKH/(\epsilon'\delta))}{N^k_h(s,a)}}\Big).
	\end{aligned}
\end{equation*}
Setting  $\epsilon' = \frac{1}{ N_h^k(s,a)} \geq \frac{1}{K}$ results in 
\begin{equation*}
	|[(\hPr_h^k-\Pr_h)V^{\star}_{h+1}](s,a;\btheta)| \leq \bigO\Big(\sqrt{\frac{SH^2}{N^k_h(s,a)}\log(\frac{SAKH}{\delta})}\Big)
\end{equation*}
The two bounds together complete the proof for the first claim.

\paragraph{Second claim:} We have $\norm{\br^k_h}\leq 1$ almost surely and $\E[\br^k_h \mid \mathcal{F}^k_h]=\br_h(s^k_h,a^k_h)$. For a fixed $(s,a,k,h) \in \cS \times \cA \times [K] \times [H]$, applying Lemma~\ref{lem:concentration-l2vector} implies that with probability at least $1-\delta'$ we have
\begin{equation*}
	\norm{(\hbr^k_{h}-\br_{h})(s,a)} \leq \bigO\Big(\sqrt{\frac{\log(d/\delta')}{N^k_h(s,a)}}\Big).
\end{equation*}
 Setting $\delta'=\delta/(SAKH)$ and applying a union bound, for all $(s,a,k,h) \in \cS \times \cA \times [K] \times [H]$, we have
 \begin{equation*}
	 \norm{(\hbr^k_{h}-\br_{h})(s,a)} \leq \bigO\Big(\sqrt{\frac{\log(dSAKH/\delta)}{N^k_h(s,a)}}\Big). 
\end{equation*}
Now consider an arbitrary $(s,a,k,h,\btheta) \in \cS \times \cA \times [K] \times [H] \times \cB(1)$, we have (by Cauchy-Schwarz)
\begin{equation*}
\begin{aligned}
		|(\hr^k_h-r_h)(s,a;\btheta)| &= |\langle \btheta, (\hbr^k_{h}-\br_{h})(s,a) |\\ &\leq \norm{\btheta} \norm{(\hbr^k_{h}-\br_{h})(s,a)} \\&\leq \bigO\Big(\sqrt{\frac{\log(dSAKH/\delta)}{N^k_h(s,a)}}\Big),
\end{aligned}
\end{equation*}
completing proof of this claim.

\paragraph{Third claim:} For a fixed $(s,a,s',k,h) \in \cS \times \cA \times \cS \times [K] \times [H]$, using empirical Bernstein inequality, with probability at least $1-\delta'$ we have
\begin{equation*}
	|(\hPr_h^k-\Pr_h)(s' \mid s,a)| \leq \bigO\Big(\sqrt{\frac{\hPr_h^k(s'\mid s,a)\log(1/\delta')}{N^k_h(s,a)}}+\frac{\log(1/\delta')}{N^k_h(s,a)}\Big)
\end{equation*}
Applying a union bound and setting $\delta'=\delta/S^2AKH$ completes the proof.
\end{proof}

The following lemma shows that the optimal value functions of $\widehat{\cM}^k_\btheta$ are close to the optimal value functions of $\cM_\btheta$ and their difference is controlled by $\tQ$ and $\tV$ computed in Algorithm~\ref{alg:tabular}.

\begin{lemma}
\label{lem:closensess-QandV-tabular}
Suppose event $E_1$ holds (defined in Lemma~\ref{lem:event1}); then, for all $(s,a,k,h,\btheta) \in \cS \times \cA \times [K] \times [H] \times \cB(1)$ we have
\begin{equation}
\label{eq:closensess-QandV-tabular}
\begin{aligned}
&|\hQ^k_h(s,a;\btheta)-Q^{\star}_h(s,a;\btheta)| \leq \tQ^k_h(s,a),\\
&|\hV^k_h(s;\btheta)-V^{\star}_h(s;\btheta)| \leq \tV^k_h(s).\\
\end{aligned}
\end{equation}
\end{lemma}
\begin{proof}[Proof of Lemma~\ref{lem:closensess-QandV-tabular}]
	We prove the lemma by backward induction on $h$. For $h=H+1$ the claim holds trivially. Now suppose that the claim is true for $(h+1)^{\mathrm{th}}$ step, we want to show that the claim is also true for $h^{\mathrm{th}}$ step. For the Q-value function we have 
\begin{equation*}
\begin{aligned}
	&|\hQ^k_h(s,a;\btheta)-Q^{\star}_h(s,a;\btheta)|\\
	&\leq \min \Big\{ \underbrace{|[(\hPr^k_h-\Pr_h)V^{\star}_{h+1}](s,a;\btheta)| + |(\hr^k_h-r_h)(s,a;\btheta)|}_{(T_1)} + \underbrace{|[\hPr^k_h(\hV^k_{h+1}-V^{\star}_{h+1})](s,a;\btheta)|}_{(T_2)} ,H \Big\}\\
	&\overset{(i)}{\leq} \min \Big\{ \beta^k_h(s,a) + [\hPr^k_h \tV^k_{h+1}](s,a) ,H \Big\} \overset{(ii)}{=} \tQ^k_h(s,a),\\
\end{aligned}
\end{equation*}
where $(i)$ follows from $T_1 \leq \beta^k_h(s,a)$ (event $E_1$) and $T_2 \leq [\hPr^k_h \tV^k_{h+1}](s,a)$ (induction hypothesis), and $(ii)$ is due to definition of $\tQ^k_h$ in Algorithm~\ref{alg:tabular}. Now for the value function we have
\begin{equation*}
\begin{aligned}
&|\hV^k_h(s;\btheta)-V^{\star}_h(s;\btheta)|\\
&=|\max_{a \in \cA} \hQ^k_h(s,a;\btheta)-\max_{a' \in \cA} \hQ^{\star}(s,a';\btheta)|\\
&\leq \max_{a \in \cA}|\hQ^k_h(s,a;\btheta)- \hQ^{\star}(s,a;\btheta)|\\
&\leq \max_{a \in \cA} \tQ^k_h(s,a) = \tV^k_h(s),\\
\end{aligned}	
\end{equation*}
which completes the induction step and consequently the proof. 
\end{proof}

Now we are ready to introduce the main lemma that shows value of $\hpi^k_\btheta$ under the true model is close to its value under empirical model. The difference is controlled by $\tQ$ and $\tV$ computed in Algorithm~\ref{alg:tabular}.

\begin{lemma}
\label{lem:main-apptabular}
Suppose event $E_1$ holds (defined in Lemma~\ref{lem:event1}); then, for all $(s,a,k,h,\btheta) \in \cS \times \cA \times [K] \times [H] \times \cB(1)$ we have
\begin{equation}
\begin{aligned}
&|\hQ^k_h(s,a;\btheta)-Q_h^{\hpi^k_\btheta}(s,a;\btheta)| \leq \alpha_h \tQ^k_h(s,a), \\ 
&|\hV^k_h(s;\btheta)-V_h^{\hpi^k_\btheta}(s;\btheta)| \leq \alpha_h \tV^k_h(s), \\ 	
\end{aligned}
\end{equation}
where $\alpha_{H+1}=1$ and $\alpha_{h}=[(1+\frac{1}{H})\alpha_{h+1}+\frac{1}{H}]$; we have $ 1 \leq \alpha_h \leq 5$ for $h \in [H]$.
\end{lemma}
\begin{proof}[Proof of Lemma~\ref{lem:main-apptabular}]
We prove the claim by backward induction on $h$. For $h=H+1$ the claim trivially holds. Now suppose that the claim is true for step $h+1$ and we want to show that it also holds for step $h$.
\begin{equation}
\label{eq:apptabular-main}
\begin{aligned}
&|\hQ^k_h(s,a;\btheta)-Q_h^{\hpi^k_\btheta}(s,a;\btheta)|\\
&\leq \min \Big\{\underbrace{|[(\hPr^k_h-\Pr_h)(V_{h+1}^{\hpi^k_\btheta}-V^{\star}_{h+1})](s,a;\btheta)|}_{(T_1)}\\
&+\underbrace{|[(\hPr_h^k-\Pr_h)V^{\star}_{h+1}](s,a;\btheta)|+|(\hr^k_h-r_h)(s,a;\btheta)|}_{(T_2)}\\
&+\underbrace{|[\hPr_h^k(\hV^k_{h+1}-V_{h+1}^{\hpi^k_\btheta})](s,a;\btheta)|}_{(T_3)}, H\Big\}\\	
\end{aligned}
\end{equation}
For the term $(T_3)$, by applying induction hypothesis we have
\begin{equation}
\label{eq:apptabular-T3}
	(T_3)\leq \alpha_{h+1}[\hPr^k_h\tV^k_{h+1}](s,a).
\end{equation}
Using event $E_1$, for the term $(T_2)$ we have
\begin{equation}
\label{eq:apptabular-T2}
(T_2)\leq 2{c} \sqrt{\frac{\min\{d,S\}H^2\iota}{\max\{N^k_h(s,a),1\}}}.
\end{equation}
It only remains to bound the term $(T_1)$; we have
\begin{equation}
\label{eq:apptabular-T1}
\begin{aligned}
(T_1) &\leq \sum_{s' \in \cS}|\hPr^k_h(s'\mid s,a)-\Pr_h(s'\mid s,a)||(V_{h+1}^{\hpi^k_\btheta}-V^{\star}_{h+1})(s')|\\
&\leq \sum_{s' \in \cS}|\hPr^k_h(s'\mid s,a)-\Pr_h(s'\mid s,a)|\Big[|(V_{h+1}^{\hpi^k_\btheta}-\hV^k_{h+1})(s')|+|(\hV^k_{h+1}-V^{\star}_{h+1})(s')|\Big]\\
&\overset{(i)}{\leq}\sum_{s' \in \cS}|\hPr^k_h(s'\mid s,a)-\Pr_h(s'\mid s,a)|(\alpha_{h+1}+1)\tV^k_{h+1}(s')\\
&\overset{(ii)}{\leq}\sum_{s' \in \cS}\Big[{c}(\sqrt{\frac{\hPr_h^k(s'\mid s,a)\iota}{\max\{N^k_h(s,a),1\}}}+\frac{\iota}{\max\{N^k_h(s,a),1\}})\Big](\alpha_{h+1}+1)\tV^k_{h+1}(s')\\
&\overset{(iii)}{\leq}\sum_{s' \in \cS}\Big[\frac{\hPr_h^k(s'\mid s,a)}{H}+\frac{{c}^2H\iota+{c}\iota}{\max\{N^k_h(s,a),1\}}\Big](\alpha_{h+1}+1)\tV^k_{h+1}(s')\\
&\leq \frac{\alpha_{h+1}+1}{H}[\hPr^k_h\tV^k_{h+1}](s,a)+2{c}^2(\alpha_{h+1}+1)\frac{H^2S\iota}{\max\{N^k_h(s,a),1\}},
\end{aligned}	
\end{equation}
where $(i)$ is due Lemma~\ref{lem:closensess-QandV-tabular} along with induction hypothesis, $(ii)$ is due to event $E_1$, and $(iii)$ is by AM-GM. Plugging equation \ref{eq:apptabular-T3}, \ref{eq:apptabular-T2}, and \ref{eq:apptabular-T1} back in \ref{eq:apptabular-main}, we get 
\begin{equation}
\begin{aligned}
&|\hQ^k_h(s,a;\btheta)-Q_h^{\hpi^k_\btheta}(s,a;\btheta)|\\
&\leq \min \Big\{ [(1+\frac{1}{H})\alpha_{h+1}+\frac{1}{h}][\hPr^k_h\tV^k_{h+1}](s,a)+2{c} \sqrt{\frac{\min\{d,S\}H^2\iota}{\max\{N^k_h(s,a)+,1\}}}\\
&+2{c}^2(\alpha_{h+1}+1)\frac{H^2S\iota}{\max\{N^k_h(s,a),1\}},H\Big\}\\
&\overset{(i)}{\leq} \min \Big\{ [(1+\frac{1}{H})\alpha_{h+1}+\frac{1}{h}][\hPr^k_h\tV^k_{h+1}](s,a)+\beta^k_h(s,a),H\Big\}\\
&\overset{(ii)}{\leq} \alpha_h \min\{[\hPr^k_h\tV^k_{h+1}](s,a)+\beta^k_h(s,a),H\}\\
&\overset{(iii)}{=}\alpha_h \tQ^k_h(s,a),
\end{aligned}
\end{equation}
where $(i)$ is by the definition of the bonus $\beta^k_h$ (we have $2{c} + 12{c}^2 \leq C$ and $(\alpha_{h+1}+1)\leq 6$), $(ii)$ is by the definition of $\alpha_h$ (note that $1 \leq \alpha_h$), and $(iii)$ is by the definition of $\tQ^k_h$ in Algorihtm~\ref{alg:tabular}. The inequality for value function follows immediately since we have
\begin{equation*}
\begin{aligned}
	&|\hV^k_h(s;\btheta)-V_h^{\hpi^k_\btheta}(s;\btheta)|\\
	&=|[\mathbb{D}_{\hpi^k_\btheta}\hQ^{\hpi^k_\btheta}_h](s;\btheta)-[\mathbb{D}_{\hpi^k_\btheta}Q^k_h](s;\btheta)|\\
	&\leq \alpha_h[\mathbb{D}_{\hpi^k_\btheta}\tQ^k_h](s)\\
	&\leq \alpha_h \max_{a \in \cA}\tQ^k_h(s,a)\\
	&= \alpha_h \tV^k_h(s).
\end{aligned}
\end{equation*}
It completes the induction step and consequently the proof of the lemma.
\end{proof}

\begin{theorem}[Similar to guarantee for UCB-VI from \cite{azar2017minimax}] 
\label{thm:apptabular-sumbonus}
For any $\delta \in (0,1]$, if we choose $\beta^k_t$ in Algorithm~\ref{alg:tabular} as in Equation~\ref{eq:bonus-tabular}; then, with probability at least $1-\delta$, we have
\begin{equation*}
	\sum_{k=1}^K \tV^k_1(s_1) \leq \bigO(\sqrt{\min\{d,S\}H^4SAK\iota}+H^3S^2A\iota^2).
\end{equation*}
\end{theorem}

\begin{proof}[Proof of Theorem~\ref{thm:apptabular-sumbonus}]
	For a fixed $k$, by definition of $\tV$ we have
	\begin{equation*}
		\tV^k_1(s_1) \leq \sum_{h=1}^H \big(\beta^k_h(s^k_h,a^k_h)+\zeta^k_h\big),
	\end{equation*}
	where $\zeta^k_h = [\hPr^k_h \tV^k_{h+1}](s^k_h,a^k_h)-\tV^k_{h+1}(s^k_{h+1})$. Summing over $k$ gives us,
	\begin{equation*}
		\sum_{k=1}^K \tV^k_1(s_1) \leq \underbrace{\sum_{k=1}^K\sum_{h=1}^H \beta^k_h(s^k_h,a^k_h)}_{(T_1)}+  \underbrace{\sum_{k=1}^K\sum_{h=1}^H \zeta^k_h}_{(T_2)} .
	\end{equation*}
	Now we bound each term separately. For the term $(T_1)$, using standard pigeonhole argument, we have
	\begin{equation*}
	\begin{aligned}
		(T_1) &= C \Big[\sum_{k=1}^K \sum_{h=1}^H \sqrt{\frac{\min\{d,S\}H^2{\iota}}{N^k_h(s^k_h,a^k_h)}}+\sum_{k=1}^K \sum_{h=1}^H\frac{H^2S{\iota}}{N^k_h(s^k_h,a^k_h)}\Big]\\
		&= C \Big[\sqrt{\min\{d,S\}H^2{\iota}}\sum_{{h,s,a}}\sum_{i=1}^{N^K_h(s,a)}\sqrt{\frac{1}{i}}+H^2S{\iota}\sum_{{h,s,a}}\sum_{i=1}^{N^K_h(s,a)}\frac{1}{i}\Big]\\
		&\leq C' \Big[\sqrt{\min\{d,S\}H^2{\iota}}\sum_{{h,s,a}}\sqrt{N^K_h(s,a)}+H^2S{\iota}\sum_{{h,s,a}}\log(KH) \Big]\\
		&\leq C'\Big[\sqrt{\min\{d,S\}H^2{\iota}}\sqrt{HSA}\sqrt{KH}+H^3S^2A{\iota}^2\Big]\\
		&\leq \bigO(\sqrt{\min\{d,S\}H^4SAK{\iota}}+H^3S^2A{\iota}^2).
	\end{aligned}
	\end{equation*}
	For the second term, note that $\zeta^k_h$ forms a martingale difference sequence; therefore, by Azuma-Hoeffding's inequality, with probability at least $1-\delta$, we have
	\begin{equation*}
		(T_2) \leq \bigO(H\sqrt{ (KH)\log(1/\delta)})=\bigO(\sqrt{H^3K\log(1/\delta)}),
	\end{equation*}
	resulting in a lower order term and completing the proof.
\end{proof}

\begin{proof}[Proof of Theorem~\ref{restate:thm:tabular} (restatement of Theroem~\ref{thm:tabular})]
By Algorithm~\ref{alg:tabular}, we have $\mathrm{out} = \argmin_{k \in [K]}\tV^k_1(s_1)$, resulting in $\tV^{\mathrm{out}}_1(s_1)\leq \frac{1}{K}\sum_{k=1}^K \tV^k_1(s_1)$. Therefore, with probability at least $1-2\delta$, for any vector $\btheta \in \cB(1)$ we have
\begin{equation*}
\begin{aligned}
	V_1^{\star}(s_1;\btheta)-V^{\hpi^\mathrm{out}_\btheta}_1(s_1;\btheta) &\leq |V_1^{\star}(s_1;\btheta)-\hV^\mathrm{out}_1(s_1;\btheta)|+|\hV^\mathrm{out}_1(s_1;\btheta)-V^{\hpi^\mathrm{out}_\btheta}_1(s_1;\btheta)|\\
	&\overset{(i)}{\leq} (1+\alpha_1)\tV^{\mathrm{out}}_1(s_1)\\
	&\leq 6\tV^{\mathrm{out}}_1(s_1) \\
	&\leq \frac{6}{K}\sum_{k=1}^K \tV^k_1(s_1)\\
	&\overset{(ii)}{\leq} \bigO(\sqrt{\min\{d,S\}H^4SA{\iota}/K}+H^3S^2A{\iota}^2/K)\\
	&\overset{(iii)}{\leq} \epsilon,
\end{aligned}
\end{equation*}
	where $(i)$ is due to Lemma~\ref{lem:closensess-QandV-tabular} and Lemma~\ref{lem:main-apptabular}, $(ii)$ is due to Theorem~\ref{thm:apptabular-sumbonus}, and $(iii)$ is due to $K \geq c_K (\min\{d,S\}H^4SA\iota'/\epsilon^2+H^3S^2A(\iota')^2/\epsilon)$ with a sufficiently large constant $c_K$. Rescaling $\delta$ completes the proof.
\end{proof}

\section{Proof for Section~\ref{sec:linear}}
\label{app:linear}
In this section we provide proofs and missing details for Section~\ref{sec:linear}.
\subsection{Reward-free algorithm for linear VMDPs}
We use slightly modified version of the reward-free algorithm introduced by \cite{wang2020reward}. The exploration phase and planning phase are displayed in Algorithm~\ref{alg:linear-exploration} and \ref{alg:linear-planning}, respectively.
\begin{algorithm}
    \caption{Reward-Free RL for Linear VMDPs: Exploration Phase}
    \label{alg:linear-exploration}
    \begin{algorithmic}[1]
    	\STATE \textbf{Hyperparameters:} Bonus coefficient $\beta$.
    	\FOR{episode $k=1,2,\dots,K$}
    		\FOR{step $h=H,H-1,\dots,1$}
    				\STATE $\tLambda^k_h = \sum_{i=1}^{k-1}{\bphi}(s^{i}_h,a^i_h){\bphi}(s^{i}_h,a^i_h)^\top + I$
    				\STATE $\tu^k_h(\cdot,\cdot) \leftarrow \min\{\beta\cdot\sqrt{{\bphi}(\cdot,\cdot)^\top(\tLambda^k_h)^{-1}{\bphi}(\cdot,\cdot)},H\}$
    				\STATE Define $\tr^k_h(\cdot,\cdot) \leftarrow \tu^k_h(\cdot,\cdot)/H$
    				\STATE $\tbw^k_h \leftarrow (\tLambda^k_h)^{-1} \sum_{i=1}^{k-1}{\bphi}(s^i_h,a^i_k) \tV^k_{h+1}(s^i_{h+1})$
    				\STATE $\tQ^k_h(\cdot,\cdot) \leftarrow \min\{(\tbw_h^k)^\top {\bphi}(\cdot,\cdot)+\tr^k_h(\cdot,\cdot)+\tu^k_h(\cdot,\cdot),H\}$
    				\STATE $\tV^k_h(\cdot) = \max_{a \in \cA} \tQ^k_h(\cdot,a)$ and $\tpi^k_h(\cdot) \leftarrow \argmax_{a \in \cA} \tQ^k_h(\cdot,a)$
    				
    		\ENDFOR

    		\STATE Observe initial state $s^k_1 \leftarrow s_1$
    		\FOR{step $h=1,2,\dots,H$}
    			\STATE Take action $a^k_h \leftarrow \tpi^k_h(s^k_h)$ and observe next state $s^k_{h+1}$
    		\ENDFOR
    	\ENDFOR
    	\STATE \textbf{Return} $\mathcal{D} \leftarrow \{(s^k_h,a^k_h)\}_{(h,k) \in [H] \times [K]}$
    \end{algorithmic} 
\end{algorithm}

\begin{algorithm}
    \caption{Reward-Free RL for Linear VMDPs: Planning Phase}
    \label{alg:linear-planning}
    \begin{algorithmic}[1]
    	\STATE \textbf{Hyperparameters:} Bonus coefficient $\beta$.
    	\STATE \textbf{Input:} Dataset $\mathcal{D}=\{(s^k_h,a^k_h)\}_{(k,h) \in [K]\times[H]}$, vector $\btheta \in \cB(1)$\\ \hspace{11mm}samples of return function $\{\br^k_h\}_{(k,h) \in [K] \times [H]}$
       		\FOR{step $h=H,H-1,\dots,1$}
    				\STATE $\widehat{\Lambda}_h = \sum_{i=1}^{K}{\bphi}(s^{i}_h,a^i_h){\bphi}(s^{i}_h,a^i_h)^\top + I$
    				\STATE $\widehat{u}_h(\cdot,\cdot) \leftarrow \min\{\beta\cdot\sqrt{{\bphi}(\cdot,\cdot)^\top(\widehat{\Lambda}_h)^{-1}{\bphi}(\cdot,\cdot)},H\}$
    				\label{line:linear_u}
    				\STATE $\hbw_h \leftarrow (\widehat{\Lambda}_h)^{-1} \sum_{i=1}^{K}{\bphi}(s^i_h,a^i_h) [\hV_{h+1}(s^i_{h+1})+\btheta^\top \br^i_h]$
    				\label{line:linear_w}
    				\STATE $\hQ_h(\cdot,\cdot) \leftarrow \min\{(\hbw_h)^\top {\bphi}(\cdot,\cdot)+\widehat{u}_h(\cdot,\cdot),H\}$
    				\label{line:linear_Q}
    				\STATE $\hV_h(\cdot) = \max_{a \in \cA} \hQ_h(\cdot,a)$ and $\hpi_h(\cdot) \leftarrow \argmax_{a \in \cA} \hQ_h(\cdot,a)$	
    		\ENDFOR 
    	\STATE \textbf{Return} $\pi_\btheta = \{\hpi_h\}_{h=1}^H$
    \end{algorithmic} 
\end{algorithm}

\subsection{Proof of Theorem~\ref{thm:linear}}

\begin{theorem}[restatement of Theorem~\ref{thm:linear}]
\label{restate:thm:linear}
There exist absolute constants $c_\beta$ and $c_K$, such that for any $\epsilon \in (0,H]$ and $\delta  \in (0,1]$, if we choose bonus coefficient $\beta = c_\beta \cdot \dlin H\sqrt{\iota}$ with $\iota=\log[\dlin d KH/\delta]$, and run the exploration algorithm (Algorithm~\ref{alg:linear-exploration}) for $K \geq c_K [\dlin^3H^6(\iota')^2/\epsilon^2]$ episodes where $\iota'=\log[\dlin d H/(\epsilon\delta)]$, then with probability at least $1-\delta$, for any $\btheta \in \cB(1)$, the output of the planning phase satisfies:
\begin{equation*}
	V_1^{\star}(s_1;\btheta)-V^{\pi_\btheta}_1(s_1;\btheta) \leq \epsilon,
\end{equation*}
where $\pi_{\btheta}$ is the output of the planning algorithm (Algorithm~\ref{alg:linear-planning}) given $\btheta$ as input. Therefore, in this case we have
\begin{equation*}
	m_{\textsc{RFE}}(\epsilon,\delta)	 \leq \bigO\Big(\dlin^3H^6(\iota')^2/\epsilon^2\Big).
\end{equation*}
\end{theorem}

In this section, we denote $\bphi^k_h := \bphi(s^k_h,a^k_h)$ for $ (k,h) \in [K] \times [H]$. For a scalar reward function $r':\cS \times \cA \rightarrow [-1,1]$ and a policy $\pi$, we use $V^\pi_h(\cdot \mid r')$ and $Q^\pi_h(\cdot,\cdot \mid r')$ to denote the value function and Q-value function for the MDP $(\cS,\cA,H,\Pr,r')$. Similarly we define the optimal value function and Q-value function denoted by $V^{\star}_h(\cdot \mid r')$ and $Q^{\star}_h(\cdot \mid r')$.

The bonus coefficient is defined to be 
\begin{equation}
\label{eq:bonus-linear2}
	\beta = c_\beta \cdot \dlin H\sqrt{\iota}
\end{equation}
where $\iota = \log [\dlin dHK/\delta]$.

We start with the following concentration lemma. 
\begin{lemma}
\label{lem:event2}
Suppose Assumption~\ref{asm:linear} holds. Let $c$ be some large absolute constant. Define event $E_2$ to be: for all $(k,h,\btheta) \in [K] \times [H] \times \cB(1)$, 
\begin{equation}
	\begin{cases}
		\bignorm{\sum_{i=1}^{k-1}\bphi^i_h\Big( \tV^k_{h+1}(s^i_{h+1})-[\Pr_h\tV^k_{h+1}](s^i_h,a^i_h) \Big)}_{({\tLambda}^k_h)^{-1}} &\leq c\cdot  H\sqrt{\dlin^2\iota},\\
		\bignorm{\sum_{i=1}^{K}\bphi^i_h\Big( \hV_{h+1}(s^i_{h+1})-[\Pr_h\hV_{h+1}](s^i_h,a^i_h) \Big)}_{({\hLambda}_h)^{-1}} &\leq c\cdot  H\sqrt{\dlin^2\iota},\\
		\bignorm{\sum_{i=1}^{K}\bphi^i_h\Big( \btheta^\top(\hbr_h-\br_h)(s^i_h,a^i_h) \Big)}_{({\hLambda}_h)^{-1}} &\leq c\cdot \sqrt{\dlin \iota},\\
		|\sum_{k=1}^K \sum_{h=1}^H [\Pr_h \tV^{k}_{h+1}](s^{k}_{h},a^{k}_{h})-\tV^k_{h+1}(s^{k}_{h})| &\leq c \cdot H^2 \sqrt{K\iota},\\
	\end{cases}
\end{equation}
where $\iota = \log [\dlin dHK/\delta]$. We have $\Pr(E_2) \geq 1-\delta$.
\end{lemma}
\begin{proof}[Proof of Lemma~\ref{lem:event2}]
	The first three inequalities follow from the standard concentration inequalities of the self-normalized process, a covering argument over the value functions or $\btheta$, and union bound. We refer readers to the proofs of Lemma B.3 in \cite{jin2020provably} or Lemma A.1 in \cite{wang2020reward} for details. The last inequality follows immediately from Azuma-Hoeffding's inequality since for a fixed $h$, $\{[\Pr_h \tV^{k}_{h+1}](s^{k}_{h},a^{k}_{h})-\tV^k_{h+1}(s^{k}_{h})\}_{k \in [K]}$ is a martingale difference sequence bounded by $H$.
\end{proof}

The following lemma shows that $\tV^k_1$ (defined in Algorithm~\ref{alg:linear-exploration}) is optimistic with respect to reward function $\tr^k$. In addition, it shows its sum over $k$ can be controlled by $\widetilde{\mathcal{O}}(\sqrt{\dlin^3H^4K})$.

\begin{lemma}
\label{lem:app_linear_1} 
Suppose Assumption~\ref{asm:linear} and event $E_2$ (defined in Lemma~\ref{lem:event2}) hold; we have
\begin{equation*}
\begin{aligned}
	&V_1^{\star}(s_1 \mid \tr^k) \leq \tV^k_1(s_1) \quad \forall k \in [K]\\
	&\sum_{k=1}^k \tV^k_1(s_1) \leq \bigO\Big(\sqrt{\dlin^3H^4K\iota^2}\Big) 	
\end{aligned}
\end{equation*}
\end{lemma}

\begin{proof}[Proof of Lemma~\ref{lem:app_linear_1}]
	Let ${\bbw}^k_h=\int \tV^k_{h+1}(s') \mathrm{d}\bmu_h(s')$; by Assumption~\ref{asm:linear}, we have 
	\begin{equation}
	\begin{aligned}
		&\norm{{\bbw}^k_h} \leq H\norm{\bmu_h(\cS)}\leq H\sqrt{\dlin}\\
		&[\Pr_h \tV^k_{h+1}](s,a) = \bphi(s,a)^\top {\bbw}^k_h \quad \forall (s,a) \in \cS \times \cA
	\end{aligned}
	\end{equation}
	For all $k,h,s,a \in [K] \times [H] \times \cS \times \cA$, we have
	\begin{equation*}
	\begin{aligned}
		&\bphi(s,a)^\top {\tbw}^k_h - [\Pr_h \tV^k_{h+1}](s,a)\\
		&=\bphi(s,a)^\top[{\tbw}^k_h - {\bbw}^k_h]\\
		&=\bphi(s,a)^\top ({\tLambda}^k_h)^{-1}\Big(\sum_{i=1}^{k-1}\bphi^i_h\tV^k_{h+1}(s^i_{h+1})- {\tLambda}^k_h {\bbw}^k_h\Big)\\
		&=\bphi(s,a)^\top ({\tLambda}^k_h)^{-1}\Big(\sum_{i=1}^{k-1}\bphi^i_h\tV^k_{h+1}(s^i_{h+1})- \sum_{i=1}^{k-1} \bphi^i_h (\bphi^i_h)^\top{\bbw}^k_h - {\bbw}^k_h\Big).
	\end{aligned}
	\end{equation*}
	Note that $(\bphi^i_h)^\top {\bbw}^k_h = [\Pr_h \tV^k_{h+1}](s^i_h,a^i_h)$. Therefore, we have
	\begin{equation*}
	\begin{aligned}
		&|\bphi(s,a)^\top {\tbw}^k_h - [\Pr_h \tV^k_{h+1}](s,a)|\\
		&=\Big|\bphi(s,a)^\top ({\tLambda}^k_h)^{-1}\Big[\sum_{i=1}^{k-1}\bphi^i_h\Big(\tV^k_{h+1}(s^i_{h+1})- [\Pr_h \tV^k_{h+1}](s^i_h,a^i_h)\Big)- {\bbw}^k_h\Big]\Big|\\
		&\leq \Big| \bphi(s,a)^\top ({\tLambda}^k_h)^{-1}\Big[\sum_{i=1}^{k-1}\bphi^i_h\Big(\tV^k_{h+1}(s^i_{h+1})- [\Pr_h \tV^k_{h+1}](s^i_h,a^i_h)\Big) \Big]\Big| + | \bphi(s,a)^\top ({\tLambda}^k_h)^{-1} {\bbw}^k_h |\\
		&\leq \bignorm{\sum_{i=1}^{k-1}\bphi^i_h\Big(\tV^k_{h+1}(s^i_{h+1})- [\Pr_h \tV^k_{h+1}](s^i_h,a^i_h)\Big)}_{({\tLambda}^k_h)^{-1}}\cdot \norm{\bphi(s,a)}_{({\tLambda}^k_h)^{-1}} + \norm{{\bbw}^k_h}_{({\tLambda}^k_h)^{-1}} \cdot \norm{\bphi(s,a)}_{({\tLambda}^k_h)^{-1}}.
	\end{aligned}
	\end{equation*}
	Note that $\norm{{\bbw}^k_h}_{({\tLambda}^k_h)^{-1}} \leq \norm{{\bbw}^k_h} \leq H\sqrt{\dlin}$ since $ {\tLambda}^k_h \succeq I$. By event $E_2$ we have $\bignorm{\sum_{i=1}^{k-1}\bphi^i_h\Big(\tV^k_{h+1}(s^i_{h+1})- [\Pr_h \tV^k_{h+1}](s^i_h,a^i_h)\Big)}_{({\tLambda}^k_h)^{-1}} \leq c\cdot H\sqrt{\dlin^2\iota}$. Plugging back, results in 
		\begin{equation}
		\label{eq:app_linear_eq1}
	\begin{aligned}
		&|\bphi(s,a)^\top {\tbw}^k_h - [\Pr_h \tV^k_{h+1}](s,a)|\\
		&\leq (H\sqrt{\dlin}+c\cdot H\sqrt{\dlin^2\iota}) \norm{\bphi(s,a)}_{({\tLambda}^k_h)^{-1}}\\
		&\leq (c_\beta \cdot H\sqrt{\dlin^2 \iota})\norm{\bphi(s,a)}_{({\tLambda}^k_h)^{-1}}\\
		&= \beta\norm{\bphi(s,a)}_{({\tLambda}^k_h)^{-1}}
	\end{aligned}
	\end{equation}
	Now we are ready to complete the proof:
	\paragraph{First claim:} we prove the claim 
	\begin{equation*}
		V_h^{\star}(s \mid \tr^k) \leq \tV^k_h(s) \quad \forall s \in \cS,
	\end{equation*}
	by backward induction on $h$. For $h=H+1$ the claim is trivial since both LHS and RHS are zero. Now suppose that we have
	\begin{equation*}
		V_{h+1}^{\star}(s \mid \tr^k) \leq \tV^k_{h+1}(s) \quad \forall s \in \cS.
	\end{equation*}
	Then, for all $s \in \cS$ we have
	\begin{equation*}
	\begin{aligned}
		V_{h}^{\star}(s \mid \tr^k) &= \max_{a \in \cA} Q^{\star}_h(s \mid \tr^k)\\
		&= \max_{a \in \cA} \{ \min \{ \tr^k_h(s,a) + [\Pr_h V^{\star}_{h+1}](s,a \mid \tr^k),H\}\}\\
		&\leq \max_{a \in \cA} \{ \min\{ \tr^k_h(s,a) + [\Pr_h \tV^k_{h+1}](s,a),H\}\}\\
		&\leq \max_{a \in \cA}\{\min\{\tr^k_h(s,a) + \bphi(s,a)^\top {\tbw}^k_h + \beta\norm{\bphi(s,a)}_{({\tLambda}^k_h)^{-1}},H\}\}\\
		&\leq \max_{a \in \cA} \tQ^k_h(s,a) = \tV^k_h(s),
	\end{aligned}
	\end{equation*}
	where the first inequality is due to induction hypothesis and the second inequality is due to Equation~\ref{eq:app_linear_eq1}. It proves the induction step and completes the induction.
	\paragraph{Second claim:} Let 
	\begin{equation*}
		\zeta^k_h = [\Pr_h \tV^{k}_{h+1}](s^{k}_{h},a^{k}_{h})-\tV^k_{h+1}(s^{k}_{h}) \quad \forall (k,h) \in [K] \times [H]
	\end{equation*}
	we have 
	\begin{equation*}
		\begin{aligned}
			\sum_{k=1}^K \tV^k_1(s^k_1) &\leq \sum_{k=1}^K \big((\tr^k_1+u^k_1)(s^k_1,a^k_1)+(\bphi^k_1)^\top {\tbw}^k_1\big)\\
			&= \sum_{k=1}^K \big((1+1/H)\beta \cdot \norm{\bphi(s,a)}_{({\tLambda}^k_1)^{-1}} +(\bphi^k_1)^\top {\tbw}^k_1\big)\\
			&\leq  \sum_{k=1}^K \big((2+1/H)\beta \cdot \norm{\bphi(s,a)}_{({\tLambda}^k_1)^{-1}} + [\Pr_1\tV^k_{2}][s^k_1,a^k_1]\big)\\
			&\leq \sum_{k=1}^K \big( \tV^k_2(s^k_2) +(2+1/H)\beta \cdot \norm{\bphi(s,a)}_{({\tLambda}^k_1)^{-1}} + \zeta^k_1 \big)
		\end{aligned}
	\end{equation*}
	By repeatedly applying the same argument we get
	\begin{equation*}
		\sum_{k=1}^K \tV^k_1(s^k_1) \leq (2+1/H)\beta \underbrace{\sum_{k=1}^K\sum_{h=1}^H \norm{\bphi(s,a)}_{({\tLambda}^k_h)^{-1}}}_{(T_1)} + \underbrace{\sum_{k=1}^K\sum_{h=1}^H \zeta^k_h}_{(T_2)}. 
	\end{equation*}
	For the term $(T_1)$ we have
	\begin{equation*}
	\begin{aligned}
		T_1 &= \sum_{k=1}^K\sum_{h=1}^H \norm{\bphi(s,a)}_{({\tLambda}^k_1)^{-1}} \\
		&\overset{(i)}{\leq} \sqrt{KH\sum_{k=1}^K\sum_{h=1}^H (\phi^k_h)^\top ({\tLambda}^k_h)(\phi^k_h) }\\
		&\overset{(ii)}{\leq} \sqrt{KH (2\dlin H\log(K))},\\
	\end{aligned}
	\end{equation*}
	where $(i)$ uses Cauchy-Schwarz, and $(ii)$ uses Lemma D.2 in \citet{jin2020provably} that implies $\sum_{k=1}^K\sum_{h=1}^H (\phi^k_h)^\top ({\tLambda}^k_h)(\phi^k_h) \leq 2\dlin H\log(K)$.
	
	For the term $(T_2)$, by the third inequality in event $E_2$, we have
	\begin{equation*}
		T_2 \leq c \cdot H^2 \sqrt{K\iota}.
	\end{equation*} 
	Plugging back in the original equation gives us
	\begin{equation*}
		\begin{aligned}
			&\sum_{k=1}^K \tV^k_1(s^k_1)\\ &\leq (2+1/H)\beta \cdot \sqrt{KH (2\dlin H\log(K))} + c \cdot H^2 \sqrt{K\iota}\\
			&\leq c' \sqrt{\dlin^3H^4K\iota^2},
		\end{aligned}
	\end{equation*}
	for some absolute constant $c'$, which completes the proof of the lemma.
\end{proof}

\begin{lemma}
\label{lem:app_linear_2}
Suppose Assumption~\ref{asm:linear} and event $E_2$ (defined in Lemma~\ref{lem:event2}) hold; Let $\hu=\{\hu_h\}_{h=1}^H$ (as defined in Line~\ref{line:linear_u} of Algorithm~\ref{alg:linear-planning}), we have
\begin{equation*}
	V^\star_1(s_1 \mid \hu/H) \leq \bigO\Big(\sqrt{\dlin^3H^4\iota^2/K}\Big)
\end{equation*}
\end{lemma}
\begin{proof}[Proof of Lemma~\ref{lem:app_linear_2}]
Note that $\hLambda_h \succeq {\tLambda}^k_h$ for all $k \in [K]$. Therefore for all $h \in [H]$ and $(s,a) \in \cS \times \cA$, we have
\begin{equation*}
	\hu_h(s,a)/H \leq \tu^k_h(s,a)/H = \tr^k_h(s,a) 
\end{equation*}
Using Lemma~\ref{lem:app_linear_1} we have
\begin{equation*}
\begin{aligned}
	K V^\star_1(s_1 \mid \hu/H) &\leq  \sum_{k=1}^K V^*_1(s_1 \mid \tr^k) \\
	&\leq  \sum_{k=1}^K \tV^k_1(s_1)\\
	&\leq \bigO\Big(\sqrt{\dlin^3H^4K\iota^2}\Big). \\
\end{aligned}
\end{equation*}
Dividing both sides by $K$ completes the proof.
\end{proof}

\begin{lemma}
\label{lem:app_linear_3}
Suppose Assumption~\ref{asm:linear} and event $E_2$ (defined in Lemma~\ref{lem:event2}) hold. For all $(s,a,h,\btheta) \in \cS \times \cA \times [H] \times \cB(1)$ we have
\begin{equation*}
	Q^*_h(s,a;\btheta) \leq \hQ_h(s,a) \leq \btheta^\top \br_h(s,a) + [\Pr_h \hV_{h+1}](s,a) + 2\hu_h(s,a).
\end{equation*}
\end{lemma}
\begin{proof}[Proof of Lemma~\ref{lem:app_linear_3}]
	First note that by Assumption~\ref{asm:linear}, we have $\br_h(s,a)=W_h\bphi(s,a)$. Define
	\begin{equation*}
		\bbw_h = \int \hV_{h+1}(s') \mathrm{d}\bmu_h(s') + \btheta^\top W_h.
	\end{equation*}
	By Assumption~\ref{asm:linear}, we have
	\begin{equation*}
	\begin{aligned}
		\norm{\bbw_h} &\leq \norm{\int \hV_{h+1}(s') \mathrm{d}\bmu_h(s')} + \norm{\btheta^\top W_h} \\ 
		&\leq H \norm{\bmu_h(\cS)} + \norm{\btheta} \norm{W_h}\\ 
		&\leq H\cdot\sqrt{\dlin}+\sqrt{\dlin} \leq 2H\sqrt{\dlin}.
	\end{aligned}
	\end{equation*}
	Therefore we have
	\begin{equation}
	\begin{aligned}
		&\norm{\bbw_h} \leq 2H\cdot\sqrt{\dlin}\\
		&[\Pr_h \hV_{h+1}](s,a)+\btheta^\top \br_h(s,a) = \bphi(s,a)^\top \bbw_h \quad \forall (s,a) \in \cS \times \cA
	\end{aligned}
	\end{equation}
	Now using similar argument in Lemma~\ref{lem:app_linear_1}, for all $(s,a,h,\btheta) \in \cS \times \cA \times [H] \times \cB(1)$ we can have
	\begin{equation*}
	\begin{aligned}
		&\big|\bphi(s,a)^\top {\hbw}_h - [\Pr_h \hV_{h+1}](s,a)-\btheta^\top \br_h(s,a)\big|\\
		&\leq \underbrace{\bignorm{\sum_{i=1}^{K}\bphi^i_h\Big( \hV_{h+1}(s^i_{h+1})-[\Pr_h\hV_{h+1}](s^i_h,a^i_h) \Big)}_{({\hLambda}_h)^{-1}}}_{(T_1)}\cdot \norm{\phi(s,a)}_{(\hLambda_h)^{-1}}\\
		&+ \underbrace{\bignorm{\sum_{i=1}^{K}\bphi^i_h\Big( \btheta^\top(\hbr_h-\br_h)(s^i_h,a^i_h) \Big)}_{({\hLambda}_h)^{-1}}}_{(T_2)}\cdot \norm{\phi(s,a)}_{(\hLambda_h)^{-1}}\\
		&+ \underbrace{\norm{{\bbw}^k_h}_{({\hLambda}_h)^{-1}}}_{(T_3)}\cdot \norm{\phi(s,a)}_{(\hLambda_h)^{-1}}\\
	\end{aligned}
	\end{equation*}
	Note that $(T_3)=\norm{{\bbw}_h}_{({\hLambda}_h)^{-1}} \leq \norm{{\bbw}_h} \leq 2H\sqrt{\dlin}$ since $ {\hLambda}_h \succeq I$. The other two terms $(T_1)$ and $(T_2)$ are both upper-bounded by $c\cdot  H\sqrt{\dlin^2\iota}$ due to event $E_2$. Plugging back results in
	\begin{equation}
	\begin{aligned}
	\label{eq:app_linear_eq2}
		&\big|\bphi(s,a)^\top {\hbw}_h - [\Pr_h \hV_{h+1}](s,a)-\btheta^\top \br_h(s,a)\big|\\
		&\leq \Big[2H\sqrt{\dlin}+2cH\sqrt{\dlin^2\iota} \Big]\norm{\phi(s,a)}_{(\hLambda_h)^{-1}}\\
		&\leq [c_\beta \cdot H\sqrt{\dlin^2 \iota}]\norm{\phi(s,a)}_{(\hLambda_h)^{-1}}\\	
		&=\beta \norm{\phi(s,a)}_{(\hLambda_h)^{-1}}.
	\end{aligned}
	\end{equation}
	Now we are ready to complete the proof of the lemma. For all $(s,a,h,\btheta) \leq \cS \times \cA \times [H] \times \cB(1)$, we have
	\begin{equation*}
	\begin{aligned}
		\hQ_h(s,a) &= \min\{\bphi(s,a)^\top\hbw_h +\hu_h(s,a) ,H\}\\
		&\leq \min\{[\Pr_h \hV_{h+1}](s,a)+ \btheta^\top \br_h(s,a) +2\beta \norm{\phi(s,a)}_{(\hLambda_h)^{-1}} ,H\}\\
		&\leq [\Pr_h \hV_{h+1}](s,a)+ \btheta^\top \br_h(s,a) + 2\min\{\beta \norm{\phi(s,a)}_{(\hLambda_h)^{-1}},H\}\\ 
		&= [\Pr_h \hV_{h+1}](s,a)+ \btheta^\top \br_h(s,a) + 2\hu_h(s,a),
	\end{aligned}
	\end{equation*}
	where the first inequality uses Equation~\ref{eq:app_linear_eq2}. It completes the proof for one side of the inequality in Lemma~\ref{lem:app_linear_3}. For the other side we prove the claim by backward induction on $h$. For $h=H+1$ we the claim is trivial. Now suppose that 
	\begin{equation*}
		Q^*_{h+1}(s,a;\btheta) \leq \hQ_{h+1}(s,a),
	\end{equation*}
	we want to prove the claim for $h$. We have
	\begin{equation*}
	\begin{aligned}
		Q^*_{h}(s,a;\btheta) &= \min\{\btheta^\top \br_h(s,a)+[\Pr_h V^*_{h+1}](s,a;\btheta),H\}\\
		&\overset{(i)}{\leq} \min\{\btheta^\top \br_h(s,a)+[\Pr_h \hV_{h+1}](s,a;\btheta),H\}\\
		&\overset{(ii)}{\leq}\min\{\bphi(s,a)^\top {\hbw}_h + \beta \norm{\phi(s,a)}_{(\hLambda_h)^{-1}} ,H\}\\
		&\leq \min\{\bphi(s,a)^\top {\hbw}_h + \min\{\beta \norm{\phi(s,a)}_{(\hLambda_h)^{-1}},H\} ,H\}\\
		&=\min\{\bphi(s,a)^\top {\hbw}_h + \hu_h(s,a) ,H\} = \hQ_h(s,a),
	\end{aligned}
	\end{equation*}
	where $(i)$ uses induction hypothesis, and $(ii)$ uses Equation~\ref{eq:app_linear_eq2}. It completes the proof of the lemma.
\end{proof}

\begin{proof}[Proof of Theorem~\ref{restate:thm:linear} (restatement of Theorem~\ref{thm:linear})] 
With probability at least $1-\delta$, event $E_2$ holds and we have
\begin{equation}
\label{eq:app_linear_eq3}
\begin{aligned}
	&\hV_1(s_1)-V^{\hpi}_1(s_1;\btheta) \\
	&= \hQ_1(s_1,\hpi_1(s_1))-Q^{\hpi}_1(s_1,\hpi_1(s_1);\btheta)\\
	&\overset{(i)}{\leq} \Big([\Pr_1 \hV_{2}](s_1,\hpi_1(s_1))+ \btheta^\top \br_1(s_1,\hpi_1(s_1)) + 2\hu_1(s_1,\hpi_1(s_1))\Big)\\
	&\hspace{6mm}-\Big(\btheta^\top \br_1(s_1,\hpi_1(s_1)) + [\Pr_1 V^{\hpi}_{2}](s_1,\hpi_1(s_1);\btheta) \Big)\\
	&=2\hu_1(s_1,\hpi_1(s_1))+\Big([\Pr_1 \hV_{2}](s_1,\hpi_1(s_1))-[\Pr_1 V^{\hpi}_{2}](s_1,\hpi_1(s_1);\btheta)\Big)\\
	&=2\hu_1(s_1,\hpi_1(s_1))+\E_{s_2 \sim \hpi}[\hV_2(s_2)-V^{\hpi}_2(s_2;\btheta)]\\
	&=\dots\\
	&=2\mathbb{E}_{\hpi}[\sum_{h=1}^H \hu_h(s_h,a_h)]\\
	&=2V_1^{\hpi}(s_1 \mid \hu),
\end{aligned}
\end{equation}
where $(i)$ is uses Lemma~\ref{lem:app_linear_3}. Therefore we have 
\begin{equation*}
\begin{aligned}
	&V_1^{\star}(s_1;\btheta)-V^{\hpi}_1(s_1;\btheta)\\
	&\overset{(i)}{\leq}  \hV_1(s_1) - V^{\hpi}_1(s_1;\btheta)\\
	&\overset{(ii)}{\leq} 2V_1^{\hpi}(s_1 \mid \hu)\\
	&\overset{(iii)}{\leq} 2V_1^{\star}(s_1 \mid \hu)\\
	&= 2H \cdot V_1^{\star}(s_1 \mid \hu/H)\\
	&\overset{(iv)}{\leq}\bigO\Big(\sqrt{\dlin^3H^6\iota^2/K}\Big)\\
	&\overset{(v)}{\leq} \epsilon,
\end{aligned}
\end{equation*}
where $(i)$ uses Lemma~\ref{lem:app_linear_3}, $(ii)$ uses Equation~\ref{eq:app_linear_eq3}, $(iii)$ uses definition of optimal value function, $(iv)$ uses Lemma~\ref{lem:app_linear_2}, and $(v)$ is due to $K \geq c_K [\dlin^3H^6(\iota')^2/\epsilon^2]$ with a sufficiently large constant $c_K$; It completes the proof.

\end{proof}

\section{Proof for Section~\ref{sec:mg}}
\label{app:mg}
In this section we provide proofs and missing details for Section~\ref{sec:mg}.
\subsection{Proof of Theorem~\ref{thm:meta-mg}}
Define $\bv^t = V^{\mu^t,\nu^t}_1(s_1)$ and note that $\E[\hbv^t]=\bv^t$. 
\begin{lemma}
\label{lem:event3}
	Define even $E_3$ to be:
\begin{equation*}
	\begin{cases}
		\norm{\frac{1}{T}\sum_{t=1}^T\bv^t-\hbv^t} \leq \bigO(\sqrt{dH^2\iota/T}), \\
		V^{\mu^t,\nu^t}_1(s_1;\btheta^t) - V^\star_1(s_1;\btheta^t) \leq  \epsilon/2 \quad \forall t \in [T].
	\end{cases}
\end{equation*}
where $\iota=\log(d/\delta)$. We have $\Pr(E_0) \geq 1-\delta$.
\end{lemma}

\begin{proof}[Proof of Lemma~\ref{lem:event3}]
We prove each claim holds with probability at least $1-\delta/2$; applying union bound completes the proof.
\paragraph{First claim.} Let $\mathcal{F}_t$ be the filtration capturing all the randomness in the algorithm before iteration $t$. We have $\E[\hbv^t \mid \mathcal{F}_t]=\bv^t$ and we also know that $\norm{\hbv^t} \leq H$ almost surely. By applying Lemma~\ref{lem:concentration-l2vector}, with probability at least $1-\delta$ we have
\begin{equation*}
	\norm{\frac{1}{T}\sum_{t=1}^T\bv^t-{\hbv^t}} \leq \bigO(\sqrt{H^2\log[d/\delta]/T}),
\end{equation*} 
which completes the proof.
\paragraph{Second claim.}
We have $K \geq m_\textsc{RFE}(\epsilon/2,\delta/2)$, therefore by probability at least $1-\delta/2$ (Definition~\ref{def:reward-free-mg}) we have
\begin{equation*}
	 V^{\mu^t,\dagger}_1(s_1;\btheta^t) - V^{\dagger,\omega^t}_1(s_1;\btheta^t) \leq \epsilon/2,
\end{equation*}
Since $(\mu^t,\omega^t)$ is the output of the planning phase. By definition of $V^\star$, $V^{\cdot,\dagger}$, and $V^{\dagger,\cdot}$, we further know that 
\begin{equation*}
\begin{aligned}
	 & V^\star_1(s_1;\btheta^t) =  \max_{\nu}V^{\dagger,\nu}_1(s_1;\btheta^t) \geq V^{\dagger,\omega^t}_1(s_1;\btheta^t)\\
	 &V^{\mu^t,\dagger}_1(s_1;\btheta^t) = \max_{\nu}V^{\mu^t,\nu}_1(s_1;\btheta^t) \geq V^{\mu^t,\nu^t}_1(s_1;\btheta^t)\\
\end{aligned}
\end{equation*}
Combining the three equations gives us,
\begin{equation*}
	V^{\mu^t,\nu^t}_1(s_1;\btheta^t) - V^\star_1(s_1;\btheta^t) \leq  V^{\mu^t,\dagger}_1(s_1;\btheta^t) - V^{\dagger,\omega^t}_1(s_1;\btheta^t) \leq \epsilon/2,
\end{equation*}
and completes the proof.
\end{proof}

\begin{lemma}
\label{lem:delta-approachability-mg}
	For any $\btheta \in \cB(1)$, we have
	\begin{equation*}
		V^*_1(s_1;\btheta) \leq \max_{\bx \in \cC} \langle \btheta,\bx \rangle + \max_{\nu}\min_{\mu}\dist(\bV^{\mu,\nu}_1(s_1),\cC)
	\end{equation*}
\end{lemma}

\begin{proof}[Proof of Lemma~\ref{lem:delta-approachability-mg}]
	Let $\alpha = \max_{\nu}\min_{\mu}\dist(\bV^{\mu,\nu}_1(s_1),\cC)$; therefore, for every max-player policy $\nu$ there exist a min-player policy $\overline{\mu}(\nu)$ such that $\dist(\bV^{\overline{\mu}(\nu),\nu}_1(s_1),\cC) \leq \alpha$. Let $\Gamma_{\cC}$ be the (Euclidean) projection operator into $\cC$. We have
	\begin{equation*}
		\begin{aligned}
			V^*_1(s_1;\btheta) &= V^{\mu^\star,\nu^*}_1(s_1;\btheta)\\
			&\leq V^{\overline{\mu}(\nu^*),\nu^*}_1(s_1;\btheta)\\
			&= \langle \btheta , \bV^{\overline{\mu}(\nu^*),\nu^*}_1(s_1) \rangle \\
			&= \langle \btheta , \bV^{\overline{\mu}(\nu^*),\nu^*}_1(s_1)  - \Gamma_{\cC}\Big[\bV^{\overline{\mu}(\nu^*),\nu^*}_1(s_1)\Big] \rangle + \langle \btheta , \Gamma_{\cC}\Big[\bV^{\overline{\mu}(\nu^*),\nu^*}_1(s_1)\Big] \rangle\\
			&\leq \norm{\btheta} \dist( \bV^{\overline{\mu}(\nu^*),\nu^*}_1(s_1),\cC)+ \max_{\bx \in \cC} \langle \btheta , \bx \rangle \\
			&\leq \alpha + \max_{\bx \in \cC} \langle \btheta , \bx \rangle,\\
		\end{aligned}
	\end{equation*}
	Recalling that $\alpha = \max_{\nu}\min_{\mu}\dist(\bV^{\mu,\nu}_1(s_1),\cC)$ completes the proof.
\end{proof}

\begin{proof}[Proof of Theorem~\ref{thm:meta-mg}]
With probability at least $1-\delta$, event $E_3$ (as in Definition~\ref{lem:event3}) holds and we have
\begin{equation*}
\begin{aligned}
\dist(\frac{1}{T}\sum_{t=1}^T\bV^{\mu^t,\nu^t}_1(s_1),\cC) &=
\dist\Big(\frac{1}{T}\sum_{t=1}^T\bv^t,\cC\Big)\\
&\overset{(i)}{=} \max_{\btheta \in \cB(1)} \Big[\langle \btheta , \frac{1}{T}\sum_{t=1}^T\bv^t)\rangle - \max_{\bx \in \cC} \langle \btheta, \bx \rangle \Big]\\
&=\max_{\btheta \in \cB(1)} \Big[\frac{1}{T}\sum_{t=1}^T\big(\langle \btheta, \hbv^t \rangle - \max_{\bx \in \cC}\langle \btheta, \bx  \rangle) + \langle \btheta, \frac{1}{T}\sum_{t=1}^T \bv^t-\hbv^t \rangle \Big]\\
&\overset{(ii)}{\leq}\max_{\btheta \in \cB(1)} \Big[\frac{1}{T}\sum_{t=1}^T\big(\langle \btheta, \hbv^t \rangle - \max_{\bx \in \cC}\langle \btheta, \bx  \rangle)\Big] + \bigO(\sqrt{dH^2\iota/T})\\
&\overset{(iii)}{\leq}\frac{1}{T}\sum_{t=1}^T\big(\langle \btheta^t, \hbv^t \rangle - \max_{\bx \in \cC}\langle \btheta^t, \bx  \rangle) + \bigO(\sqrt{H^2/ T})+ \bigO(\sqrt{dH^2\iota/T})\\
&\overset{(iv)}{\leq} \max_{\nu}\min_{\mu}\dist(\bV^{\mu,\nu}_1(s_1),\cC) + \frac{1}{T}\sum_{t=1}^T\big(\langle \btheta^t, \hbv^t \rangle - \bV^*_1(s_1;\btheta^t)\big) +\bigO(\sqrt{dH^2\iota/T})\\
&\overset{(v)}{\leq} \max_{\nu}\min_{\mu}\dist(\bV^{\mu,\nu}_1(s_1),\cC) + \epsilon/2+ \frac{1}{T}\sum_{t=1}^T\big(\langle \btheta^t, \hbv^t \rangle  -\bV^{\mu^t,\nu^t}_1(s_1;\btheta^t)\big) +\bigO(\sqrt{dH^2\iota/T})\\
&= \max_{\nu}\min_{\mu}\dist(\bV^{\mu,\nu}_1(s_1),\cC) + \epsilon/2+ \frac{1}{T}\sum_{t=1}^T\langle \btheta^t, \hbv^t -\bv^t \rangle +\bigO(\sqrt{dH^2\iota/T})\\
&\overset{(vi)}{\leq}\max_{\nu}\min_{\mu}\dist(\bV^{\mu,\nu}_1(s_1),\cC) + \epsilon/2 +\bigO(\sqrt{dH^2\iota/T})\\
&\overset{(vii)}{\leq} \max_{\nu}\min_{\mu}\dist(\bV^{\mu,\nu}_1(s_1),\cC) + \epsilon
\end{aligned}	
\end{equation*}
where $(i)$ is by Equation~\ref{eq:dist-fenchel}, $(ii)$ is by first inequality in event $E_3$ together with Cauchy-Schwarz, $(iii)$ is by guarantee of OGA in Theorem~\ref{thm:oga}, $(iv)$ is by Lemma~\ref{lem:delta-approachability-mg}, $(v)$ is by second inequality in event $E_3$, $(vi)$ is by first inequality in event $E_3$ together with Cauchy-Schwarz, and finally $(vii)$ is by setting $T \geq c \big(dH^2\iota/\epsilon^2\big)$ for large enough constant $c$, completing the proof.
\end{proof}

\subsection{Proof of Theorem~\ref{thm:tabular-mg}}
\subsubsection{Algorithm}
\paragraph{Exploration phase.}
Similar to Algorithm~\ref{alg:tabular}, we use VI-Zero proposed by \cite{liu2020sharp} with different choice of hyperparameters. The pseudo-code is provided in Algorithm~\ref{alg:tabular-mg}.
\paragraph{Planning phase.} In the planning phase, given $\btheta \in \cB(1)$ as input we can use any planning algorithm for $\widehat{\cG}_\btheta = (\cS,\cA,\cB,H,\hProut,\langle \btheta,\hbr \rangle)$ where $\hbr$ is empirical estimate of $\br$ using collected samples $\{\br^k_h\}$. One such algorithm could be Nash value iteration (e.g. see Algorithm 5 in \cite{liu2020sharp}) that computes Nash equilibrium policy for a \emph{known} model.
\begin{algorithm}
    \caption{VI-Zero for VMGs: Exploration Phase}
    \label{alg:tabular-mg}
    \begin{algorithmic}[1]
    	\STATE \textbf{Hyperparameters:} Bonus $\beta_t$.
    	\STATE \textbf{Initialize:} 
    	for all $(s,a,b,h)\in \cS \times \cA \times \cB \times [H]$: $\tQ_h(s,a,b)\leftarrow H$ and $N_h(s,a,b) \leftarrow 0$,    	
    	\STATE \hspace{14.5mm} for all $(s,a,b,h,s') \in \cS \times \cA \times \cB \times [H] \times \cS$: $N_h(s,a,b,s') \leftarrow 0$,
		\STATE \hspace{14.5mm} $\Delta \leftarrow 0$.
    	\FOR{episode $k=1,2,\dots,K$}
    		\FOR{step $h=H,H-1,\dots,1$}
    			\FOR{state-action pair $(s,a,b) \in \cS \times \cA \times \cB$}
    				\STATE $t \leftarrow N_h(s,a,b)$.
    				\IF{$t>0$}
    					\STATE $\tQ_h(s,a,b) \leftarrow \min\{[\hPr_h\tV_{h+1}](s,a,b)+\beta_t,H\}$.
    				\ENDIF
    			\ENDFOR
    			\FOR{state $s \in \cS$}
    				\STATE $\tV_h(s) \leftarrow \max_{(a,b) \in \cA \times \cB}\tQ_h(s,a,b)$ and  $\pi_h(s) \leftarrow \argmax_{(a,b) \in \cA \times \cB}\tQ_h(s,a,b)$
    			\ENDFOR
    		\ENDFOR
    		\IF{ $\tV(s_1) \leq \Delta$}
    			\STATE $\Delta \leftarrow \tV(s_1)$ and $\hProut \leftarrow \hPr_h$
    		\ENDIF
    		\FOR{step $h=1,2,\dots,H$}
    			\STATE Take action $(a_h,b_h) \leftarrow \pi_h(s_h)$ and observe next state $s_{h+1}$
    			\STATE Update $N_h(s_h,a_h,b_h) \leftarrow N_h(s_h,a_h,b_h) + 1$
    			\STATE Update $N_h(s_h,a_h,b_h,s_{h+1}) \leftarrow N_h(s_h,a_h,b_h,s_{h+1}) + 1$
    			\STATE $\hPr_h(\cdot \mid s_h,a_h,b_h) \leftarrow N_h(s_h,a_h,b_h,\cdot) / N_h(s_h,a_h,b_h)$
    		\ENDFOR
    	\ENDFOR
    	\STATE \textbf{Return} $\hProut$
    \end{algorithmic} 
\end{algorithm}
\subsubsection{Proof of Theorem~\ref{thm:tabular-mg}}
Proof is almost identical to proof of Theorem~\ref{thm:tabular} provided in Appendix~\ref{app:tabular}; therefore, we only provide the statement for the main lemmas without proof.

Let $\hPr^k$ and $\hbr^k$ be our empirical estimates of the transition and the return vectors at the beginning of the $k^\mathrm{th}$ episode in Algorithm~\ref{alg:tabular-mg} and define $\widehat{\cG}^k=(\cS,\cA,\cB,H,\hPr^k,\hbr^k)$. We use $N^k_h(s,a,b)$ to denote the number of times we have visited state-action $(s,a,b)$ in step $h$ before $k^\mathrm{th}$ episode in Algorithm~\ref{alg:tabular-mg}. We use superscript $k$ to denote variable corresponding to episode $k$; in particular, $(s^k_1,a^k_1,b^k_1,\dots,s^k_H,a^k_H,b^k_H)$ is the trajectory we have visited in the $k^\text{th}$ episode.

For any $\btheta \in \cB(1)$, let $\widehat{\cG}^k_\btheta$ be the scalarized VMG using vector $\btheta$ (defined in Section~\ref{sec:mg}). We use $\hV^k(\cdot;\btheta)$, $\hQ^k(\cdot,\cdot,\cdot;\btheta)$, and $(\hmu^k_\btheta,\hnu^k_\btheta)=(\hmu^k(\cdot;\btheta),\hnu^k(\cdot;\btheta))$ to denote the optimal value function, optimal Q-value function, and Nash equilibrium policy of $\widehat{\cG}^k_\btheta$ respectively. Therefore, we have
\begin{equation}
\label{eq:bellman-app-tabular-mg}
\begin{aligned}
	&\hQ^k_h(s,a,b;\btheta)=[\hPr^k_h\hV^k_{h+1}](s,a,b;\btheta)+\hr^k_h(s,a,b;\btheta),\\
	&\hV^k_h(s;\btheta)=\min_{\mu}\max_{\nu}[\mathbb{D}_{\mu \times \nu}\hQ^k_h](s;\btheta),\\
	&\hV^k_h(s;\btheta)=[\mathbb{D}_{\hmu^k_\btheta \times \hnu^k_\btheta}\hQ^k_h](s;\btheta).\\
\end{aligned}
\end{equation}

\begin{theorem}[restatement of Theorem~\ref{thm:tabular-mg}]
\label{restate:thm:tabular-mg}
There exist absolute constants $c_\beta$ and $c_K$, such that for any $\epsilon \in (0,H]$, $\delta  \in (0,1]$, if we choose bonus $\beta_t = c_\beta\big(\sqrt{\min\{d,S\}H^2\iota/t}+H^2S\iota/t\big)$ where $\iota=\log[dSABKH/\delta]$, and run the exploration phase (Algorithm~\ref{alg:tabular-mg}) for $K \geq c_K\big(\min\{d,S\}H^4SAB\iota'/\epsilon^2+H^3S^2AB(\iota')^2/\epsilon\Big)$ episodes where $\iota'=\log[dSABH/(\epsilon\delta)]$, then with probability at least $1-\delta$, the algorithm satisfies for all $\btheta \in \cB(1)$
\begin{equation*}
	 V^{\mu_\btheta,\dagger}_1(s_1;\btheta) - V^{\dagger,\nu_{\btheta}}_1(s_1;\btheta) = [V^{\mu_\btheta,\dagger}_1(s_1;\btheta)-V^\star_1(s_1;\btheta)]+[V^\star_1(s_1;\btheta)-V^{\dagger,\nu_{\btheta}}_1(s_1;\btheta)] \leq \epsilon,
\end{equation*}
where $(\mu_\btheta,\nu_\btheta)$ is the output of any planning algorithm (e.g., Nash value iteration) for the Markov game $\widehat{\cG}^\mathrm{out}_\btheta$. Therefore, we have 
\begin{equation*}
m_{\textsc{RFE}}(\epsilon,\delta)	 \leq \bigO\Big(\frac{\min\{d,S\}H^4SAB\iota'}{\epsilon^2}+\frac{H^3S^2AB(\iota')^2}{\epsilon}\Big).
\end{equation*}
\end{theorem}

The bonus for episode $k$ can be written as
\begin{equation}
\label{eq:bonus-tabular-mg}
	\beta^k_h(s,a,b) = c_\beta  \Big(\sqrt{\frac{\min\{d,S\}H^2\iota}{\max\{N^k_h(s,a,b),1\}}}+\frac{H^2S\iota}{\max\{N^k_h(s,a,b),1\}} \Big),
\end{equation}
where $\iota = \log[dSABKH/\delta]$ and $c_\beta$ is some large absolute constant.

We start with the concentration lemma similar to Lemma~\ref{lem:event1}.

\begin{lemma}
\label{lem:event4}
Let $c$ be some large absolute constant. Define event $E_4$ to be: for all $(s,a,b,s',h) \in \cS \times \cA \times \cB \times \cS \times [H]$, $k \in [K]$, and $ \btheta \in \cB(1)$,
\begin{equation}
	\begin{cases}
		|[(\hPr_h^k-\Pr_h)V^{\star}_{h+1}](s,a,b;\btheta)|&\leq {c}\sqrt{\frac{\min\{d,S\}H^2\iota}{\max\{N^k_h(s,a,b),1\}}},\\
		|(\hr^k_h-r_h)(s,a,b;\btheta)|&\leq {c}\sqrt{\frac{\iota}{\max\{N^k_h(s,a,b),1\}}},\\
		|(\hPr_h^k-\Pr_h)(s' \mid s,a,b)|&\leq {c}\Big(\sqrt{\frac{\hPr_h^k(s'\mid s,a,b)\iota}{\max\{N^k_h(s,a,b),1\}}}+\frac{\iota}{\max\{N^k_h(s,a,b),1\}}\Big),
	\end{cases}
\end{equation}
where $\iota = \log[dSABKH/\delta]$. We have $\Pr(E_4) \geq 1-\delta$.
\end{lemma}

Similar to Lemma~\ref{lem:closensess-QandV-tabular}, the following lemma shows that the optimal value functions of $\widehat{\cG}^k_\btheta$ are close to the optimal value functions of $\cG_\btheta$ and their difference is controlled by $\tQ$ and $\tV$ computed in Algorithm~\ref{alg:tabular-mg}.

\begin{lemma}
\label{lem:closensess-QandV-tabular-mg}
Suppose event $E_4$ holds (defined in Lemma~\ref{lem:event4}); then, for all $(s,a,b,k,h,\btheta) \in \cS \times \cA \times \cB \times [K] \times [H] \times \cB(1)$ we have
\begin{equation}
\label{eq:closensess-QandV-tabular-mg}
\begin{aligned}
&|\hQ^k_h(s,a,b;\btheta)-Q^{\star}_h(s,a,b;\btheta)| \leq \tQ^k_h(s,a,b),\\
&|\hV^k_h(s;\btheta)-V^{\star}_h(s;\btheta)| \leq \tV^k_h(s).\\
\end{aligned}
\end{equation}
\end{lemma}

Similar to Lemma~\ref{lem:main-apptabular}, now we are ready to introduce the main lemma that shows value of $\hpi^k_\btheta$ under the true model is close to its value under empirical model. The difference is controlled by $\tQ$ and $\tV$ computed in Algorithm~\ref{alg:tabular-mg}.

\begin{lemma}
\label{lem:main-apptabular-mg}
Suppose event $E_4$ holds (defined in Lemma~\ref{lem:event4}); then, for all $(s,a,b,k,h,\btheta) \in \cS \times \cA \times \cB \times [K] \times [H] \times \cB(1)$ we have
\begin{equation}
\begin{aligned}
&|\hQ^k_h(s,a,b;\btheta)-Q_h^{\dagger,\hnu^k_\btheta}(s,a,b;\btheta)| \leq \alpha_h \tQ^k_h(s,a,b), \\ 
&|\hV^k_h(s;\btheta)-V_h^{\dagger,\hnu^k_\btheta}(s;\btheta)| \leq \alpha_h \tV^k_h(s), \\ 	
\end{aligned}
\end{equation}
and
\begin{equation}
\begin{aligned}
&|\hQ^k_h(s,a,b;\btheta)-Q_h^{\hmu^k_\btheta,\dagger}(s,a,b;\btheta)| \leq \alpha_h \tQ^k_h(s,a,b), \\ 
&|\hV^k_h(s;\btheta)-V_h^{\hmu^k_\btheta,\dagger}(s;\btheta)| \leq \alpha_h \tV^k_h(s), \\ 	
\end{aligned}
\end{equation}
where $\alpha_{H+1}=1$ and $\alpha_{h}=[(1+\frac{1}{H})\alpha_{h+1}+\frac{1}{H}]$; we have $ 1 \leq \alpha_h \leq 5$ for $h \in [H]$.
\end{lemma}

Similar to Lemma~\ref{thm:apptabular-sumbonus}, we can bound the uncertainty using the following lemma. 
\begin{theorem}
\label{thm:apptabular-sumbonus-mg}
For any $\delta \in (0,1]$, if we choose $\beta^k_t$ in Algorithm~\ref{alg:tabular-mg} as in Equation~\ref{eq:bonus-tabular-mg}; then, with probability at least $1-\delta$, we have
\begin{equation*}
	\sum_{k=1}^K \tV^k_1(s_1) \leq \bigO(\sqrt{\min\{d,S\}H^4SABK\iota}+H^3S^2AB\iota^2).
\end{equation*}
\end{theorem}

\begin{proof}[Proof of Theorem~\ref{restate:thm:tabular-mg} (restatement of Theorem~\ref{thm:tabular-mg})]
By Algorithm~\ref{alg:tabular-mg}, we have $\mathrm{out} = \argmin_{k \in [K]}\tV^k_1(s_1)$, resulting in $\tV^{\mathrm{out}}_1(s_1)\leq \frac{1}{K}\sum_{k=1}^K \tV^k_1(s_1)$. Therefore, with probability at least $1-2\delta$, for any vector $\btheta \in \cB(1)$ we have
\begin{equation*}
\begin{aligned}
	V^{\hmu^{\mathrm{out}}_\btheta,\dagger}_1(s_1;\btheta) - V^{\dagger,\hnu^{\mathrm{out}}_{\btheta}}_1(s_1;\btheta)  &\leq |V^{\hmu^{\mathrm{out}}_\btheta,\dagger}_1(s_1;\btheta)-\hV^\mathrm{out}_1(s_1;\btheta)|+|\hV^\mathrm{out}_1(s_1;\btheta)-V^{\dagger,\hnu^{\mathrm{out}}_{\btheta}}_1(s_1;\btheta)|\\
	&\overset{(i)}{\leq} 2\alpha_1\tV^{\mathrm{out}}_1(s_1)\\
	&\leq 10\tV^{\mathrm{out}}_1(s_1) \\
	&\leq \frac{10}{K}\sum_{k=1}^K \tV^k_1(s_1)\\
	&\overset{(ii)}{\leq} \bigO(\sqrt{\min\{d,S\}H^4SAB{\iota}/K}+H^3S^2AB{\iota}^2/K)\\
	&\overset{(iii)}{\leq} \epsilon,
\end{aligned}
\end{equation*}
where $(i)$ is due to Lemma~\ref{lem:main-apptabular-mg}, $(ii)$ is due to Theorem~\ref{thm:apptabular-sumbonus-mg}, and $(iii)$ is due to $K \geq c_K (\min\{d,S\}H^4SAB\iota'/\epsilon^2+H^3S^2AB(\iota')^2/\epsilon)$ with a sufficiently large constant $c_K$. Rescaling $\delta$ completes the proof.
\end{proof}

\section{Auxiliary tools}
\label{app:aux}

\begin{lemma}[Hoeffding type inequality for norm-subGaussian, Corollary 7 in \cite{jin2019short}] \label{lem:concentration-l2vector}
Let $\{\bX_t\}_{t \in [T]}$ be a $d$-dimensional vector-valued random variable. Consider filtration $\{\mathcal{F}_t\}_{t \in [T]}$ and define $\E_t[\cdot]=\E[\cdot \mid \mathcal{F}_t]$. If $\norm{\bX_t} \leq R$ almost surely, then it holds with probability at least $1-\delta$,
\begin{equation*}
	\norm{\sum_{t=1}^T \bX_t - \E_{t-1}[\bX_t]} \leq \bigO(R\sqrt{T\log[d/\delta]}).
\end{equation*}
	
\end{lemma}

\end{document}